\RequirePackage{fix-cm}
\documentclass[twocolumn]{svjour3}          
\smartqed  
\usepackage{times}
\usepackage{epsfig}
\usepackage{graphicx}
\usepackage{amsmath}
\usepackage{amssymb}
\usepackage{url}
\usepackage{subfigure}
\usepackage{algorithmic}
\usepackage{algorithm}
\usepackage{enumitem}
\usepackage[mathscr]{eucal}
\usepackage{booktabs}
\usepackage{tabularx}
\usepackage{epstopdf}
\usepackage[all]{xy}
\usepackage{multirow}

\newcommand{\w}{\mathbf{w}}
\newcommand{\x}{\mathbf{x}}
\newcommand{\y}{\mathbf{y}}
\newcommand{\z}{\mathbf{z}}

\newcommand{\A}{\mathbf{A}}
\newcommand{\B}{\mathbf{B}}
\newcommand{\C}{\mathbf{C}}
\def\I{\mathbf{I}}
\newcommand{\Y}{\mathbf{Y}}
\newcommand{\X}{\mathbf{X}}
\newcommand{\Z}{\mathbf{Z}}
\newcommand{\W}{\mathbf{W}}
\newcommand{\U}{\mathbf{U}}
\newcommand{\V}{\mathbf{V}}
\newcommand{\Q}{\mathbf{Q}}
\newcommand{\D}{\mathbf{D}}
\usepackage{color}
\def\red#1{\textcolor{red}{#1}}
\def\blue#1{\textcolor{blue}{#1}}
\renewcommand{\algorithmicrequire}{\textbf{Input:}}
\renewcommand{\algorithmicensure}{\textbf{Output:}}

\newcommand{\HH}{\mathbf{H}}

\newcommand{\Lambdaa}{\mathbf{\Lambda}}

\long\def\comment#1{}

\newcommand{\tr}{\mathrm{tr}}

\journalname{International Journal of Computer Vision}
\begin{document}

\title{Multi-label Learning with Missing Labels using Mixed Dependency Graphs \thanks{This work is supported by Tencent AI Lab. The participation of Bernard Ghanem is supported by the King Abdullah University of Science and Technology (KAUST) Office of Sponsored Research.  
The participation of Siwei Lyu is partially supported by National Science Foundation National Robotics Initiative (NRI) Grant (IIS-1537257) and National Science Foundation of China Project Number 61771341.
}
}


\author{Baoyuan Wu  \and
        Fan Jia \and 
        Wei Liu \and 
        Bernard Ghanem \and 
        Siwei Lyu}


\institute{Baoyuan Wu \at
              Tencent AI Lab, Shenzhen 518000, China  \\
              \email{wubaoyuan1987@gmail.com}           
           \and
           Fan Jia \and Wei Liu \at
              Tencent AI Lab, Shenzhen 518000, China 
           \and
           Bernard Ghanem \at KAUST, Thuwal 23955, Saudi Arabia
           \and
           Siwei Lyu \at 
           SUNY-Albany, 1400 Washington Ave, Albany, NY 12222, USA
}


\maketitle

\begin{abstract}
This work focuses on the problem of multi-label learning with missing labels (MLML), which aims to label each test instance with multiple class labels given training instances that have an incomplete/partial set of these labels (i.e. some of their labels are missing). The key point to handle missing labels is propagating the label information from the provided labels to missing labels, through a dependency graph that each label of each instance is treated as a node. 
We build this graph by utilizing different types of label dependencies. 
Specifically, the instance-level similarity is served as undirected edges to connect the label nodes across different instances and the semantic label hierarchy is used as directed edges to connect different classes.
This base graph is referred to as the mixed dependency graph, as it includes both undirected and directed edges. 
Furthermore, we present another two types of label dependencies to connect the label nodes across different classes. 
One is the class co-occurrence, which is also encoded as undirected edges. Combining with the above base graph, we obtain a new mixed graph, called MG-CO (mixed graph with co-occurrence). 
The other is the sparse and low rank decomposition of the whole label matrix, to embed high-order dependencies over all labels. 
Combining with the base graph, the new mixed graph is called as MG-SL (mixed graph with sparse and low rank decomposition). 
Based on MG-CO and MG-SL, we further propose two convex transductive formulations of the MLML problem, denoted as MLMG-CO and MLMG-SL respectively. 
In both formulations, the instance-level similarity is embedded through a quadratic smoothness term, while the semantic label hierarchy is used as a linear constraint. In MLMG-CO, the class co-occurrence is also formulated as a quadratic smoothness term, while the sparse and low rank decomposition is incorporated into MLMG-SL, through two additional matrices (one is assumed as sparse, and the other is assumed as low rank) and an equivalence constraint between the summation of this two matrices and the original label matrix. 
Interestingly, two important applications, including image annotation and tag based image retrieval, can be jointly handled using our proposed methods. 
Experimental results on several benchmark datasets show that our methods lead to significant improvements in performance and robustness to missing labels over the state-of-the-art methods.
\end{abstract}

\section{Introduction}
\label{sec: 1 introduction}

\begin{figure*}[phtb]
\centering
\includegraphics[width=1\textwidth, height=3.15in]{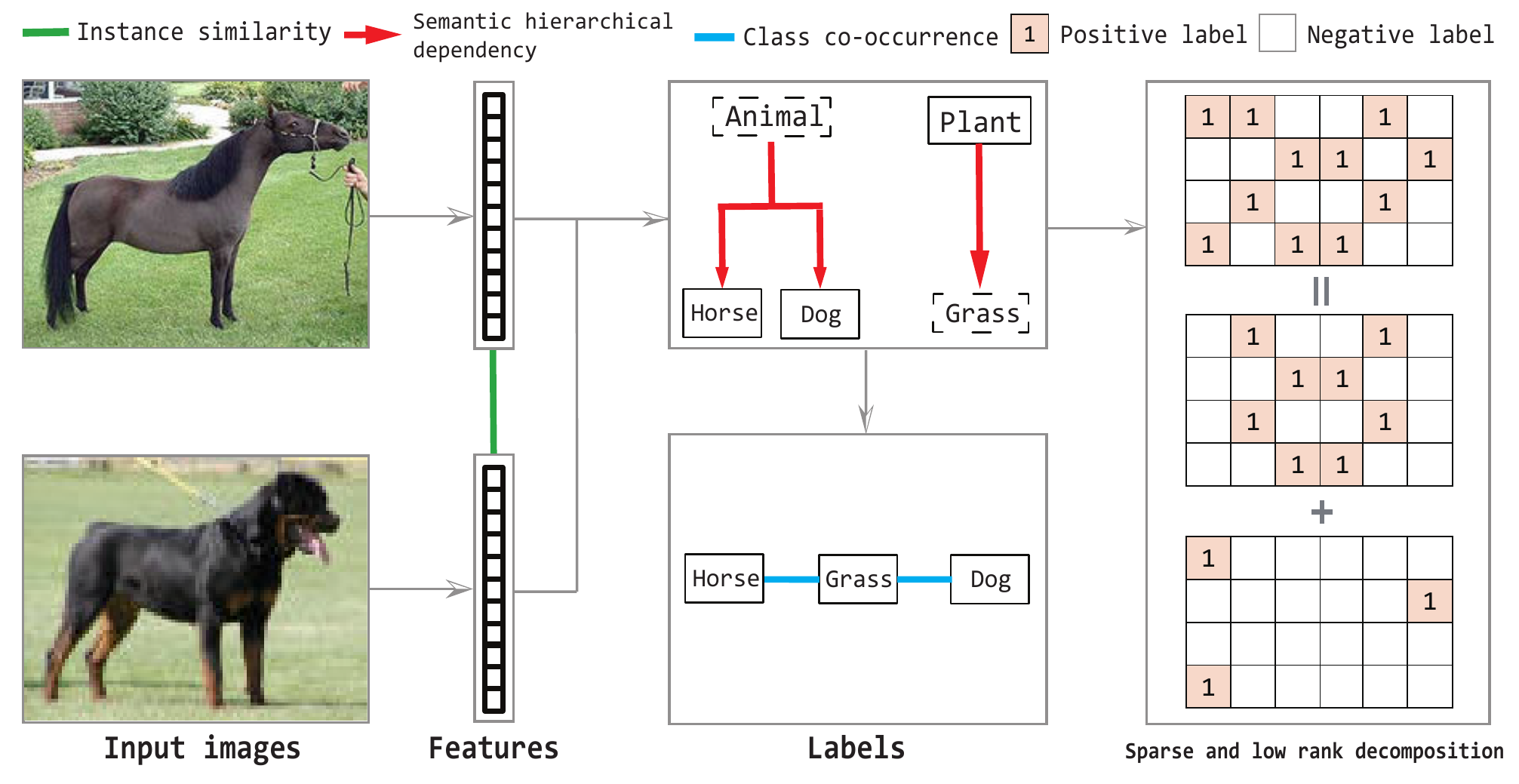}
\caption{
    The left column includes two example images from the ESP Game \cite{esp-game-2004} dataset, and their corresponding features and labels are shown in other columns.
    The solid box denotes a provided label, while the dashed box indicates a missing label.
The red (semantic hierarchical dependency), green (instance similarity),  and blue (class co-occurrence) edges constitute the mixed graph with co-occurrence (MG-CO); 
The red, green edges and the sparse and low rank decomposition of the whole label matrix constitute the mixed graph with sparse and low rank decomposition (MG-SL).
   }
\label{fig1: motivation}
\end{figure*}

In machine learning, multi-label learning refers to the setting where each data item can be associated to multiple classes simultaneously.
For example, in image annotation, an image can be annotated using several tags; in document topic analysis, a document can be associated with multiple topics. 
Although there are several multi-label learning methods in the literature \cite{mlknn-pr-2007}\cite{multilabel-review-tkde-2014}, most of these require complete labelling of training examples, {\it i.e.}, for every pair of training example and class label, their association needs to be provided.

However, complete labelling is usually infeasible in practice. Most training instances are only partially labelled, with some or all of the labels not provided/missing. Let us consider the task of large-scale image annotation, where the number of classes/tags is large ({\it e.g.}, using labels of ImageNet \cite{imagenet-cvpr-2009}). Practically, a human annotator can only consider to annotate each training image with a subset of a potentially large and diverse set of tags. Furthermore, in many cases, due to the semantic similarities in the tags, some tags are typically left unchecked, {\it e.g.}, an image tagged with ``German Shepherd" may usually not be tagged also with ``Dog".
Such a learning setting is referred to as the { \it multi-label learning with missing labels} (MLML) problem \cite{my-icpr-2014,LEML-ICML-2014}.

As labels are usually related by semantic meanings or co-occurrences,
the key to learning from missing labels is a
good model to represent label dependency. One widely used model for label dependency is an undirected graph, through which the label information can be propagated among different instances and among different classes.
For example, the label dependency between a pair of labels, such as instance similarity and class co-occurrence can be represented using such a graph (see green and blue edges in Fig. \ref{fig1: motivation}).
However, as stated in \cite{my-icpr-2014,my-pr-2015}, the class co-occurrence derived from training labels can be inaccurate and biased when many missing labels exist. 
One alleviation method is to estimate co-occurrence relations from an auxiliary and possibly more comprehensive source (such as Wikipedia) \cite{crbm-mlml-2015}. 
Another alternative is utilizing a class dependency that is independent of the provided labels. One widely used dependency in multi-label learning is the low rank assumption that the rank of the label matrix, where one row corresponds to one class, and each column indicates one instance, should be smaller than the number of rows ({\it i.e.}, classes). Although this assumption has been successfully used in many multi-label models \cite{multilabel-compressed-sensing-nips-2012,LEML-ICML-2014}, as indicated in \cite{multilabel-low-rank-sparse-kdd-2016}, the low rank assumption is difficult to be fully satisfied due to the existence of tail labels ({\it i.e.}, the rare labels that occur in very few instances, thus they are difficult to be represented by the linear combinations of other labels). 
Instead, the sparse and low-rank decomposition that has been successfully used in other applications like image alignment \cite{image-alignment-pami-2012} or visual tracking \cite{tianzhu-sparse-coding-eccv-2012} can be used in multi-label learning, to assume that the label matrix can be decomposed to the addition of one sparse and one low-rank matrices.  
Compared to the pure low rank assumption, this decomposition is more flexible to ensure the validity of the low rank assumption in practical multi-label problems. 
In this work we propose to combine the instance-level similarity with the class co-occurrence, or the sparse and low rank decomposition respectively.

The semantic dependency between two classes, such as ``animal$\rightarrow$horse" and ``plant$\rightarrow$grass" as shown in Fig. \ref{fig1: motivation}, can foster further label dependencies and improve label predictions in the test. To handle this requirement, a new set of constraints is introduced to require that {\it the label score ({\it e.g.}, the presence probability) of the parent class cannot be lower than that of its child class}. This is traditionally referred to as the {\bf semantic hierarchical constraint} \cite{bi-wei-icml-2011,my-iccv-2015}. The undirected graph (with instance similarity and class co-occurrence edges or the global sparse and low rank decomposition) cannot guarantee that the final label predictions will satisfy all semantic hierarchy constraints. To address this problem, we add semantic dependencies into the graph as directed edges, thus, resulting in an overall mixed dependency graph that encourages (or enforces) three types of label dependencies. The graph embedding the class co-occurrence is referred to as {\it mixed graph with co-occurrence} (MG-CO), while the one with the sparse and low rank decomposition is denoted as  {\it mixed graph with sparse and low rank decomposition} (MG-SL). 
Please refer to Fig. \ref{fig1: motivation} for an example of these models.

The goal of this work is to learn from partially labeled training instances and to correctly predict the labels of testing instances that satisfy the semantic hierarchical constraints.
Motivated by \cite{my-icpr-2014,my-pr-2015}, a discrete objective function is formulated to simultaneously encourage consistency between predicted and ground truth labels and encode traditional label dependencies (instance similarity with class co-occurrence or with sparse and low rank decomposition). Whereas, semantic hierarchical constraints are incorporated as hard linear constraints in the matrix optimization. The discrete problem is further relaxed to a convex problem, which is solved using ADMM \cite{admm-boyd-2011}.

\vspace{-6pt}\paragraph{Contributions:} \textbf{(1)} We address the MLML problem by using a mixed dependency graph to encode a network of label dependencies: instance similarity, class co-occurrence or sparse and low rank decomposition, as well as semantic hierarchical constraint. \textbf{(2)} Learning on the mixed dependency graph is formulated as a linearly constrained convex matrix optimization problem that is amenable to efficient solvers. \textbf{(3)} We conduct extensive experiments on the task of image annotation to show the superiority of our method in comparison to the state-of-the-art. \textbf{(4)} We augment labelling of several widely used datasets, including Corel 5k \cite{corel5k-eccv-2002}, ESP Game \cite{esp-game-2004}, IAPRTC-12 \cite{iaprtc-12-data-2006} and MediaMill \cite{mediamill-data-2006}, with a  semantic hierarchy drawn from Wordnet \cite{wordnet-1998}. This ground truth augmentation will be made publicly available to enable further researches on the MLML problem in computer vision.

Compared to the previous conference version of this work \cite{my-iccv-2015}, the additional novelties in this manuscript are threefold.  
\textbf{(1)} We adopt the CNN extracted features on ESP Game and IAPRTC-12 of which the original images are available, and the experimental performances are significantly improved compared to the one using traditional features. 
\textbf{(2)} The sparse and low rank decomposition is utilized to provide an alternative to the class co-occurrence, leading to further performance improvements.
\textbf{(3)} More detailed experimental comparisons are provided to evaluate the influences of different label dependencies. 
\textbf{(4)} The experimental results of image retrieval are added.

\section{Related Work} 
\label{sec: 2 related work}

In the literature of multi-label learning, the previous works that are designed to handle missing labels can be generally partitioned into four categories.
First, the missing labels are directly treated as negative labels,  including \cite{semi-multi-label-sdm-2008,well-multi-label-weak-2010,bucak-multi-incomplete-2011,fasttag-icml-2013,Agrawal-ml-million-label-www-2013,hash-multi-label-eccv-2014a,hash-multi-label-eccv-2014b,multilabel-link-prediction-aaai2015}.
Common to these methods is that the label bias is brought into the objective function. 
As a result, their performance is greatly affected when massive ground-truth positive labels are initialized as negative labels.
Second, filling in missing labels is treated as a matrix completion (MC) problem, including \cite{MC-nips-2010,MC-Pos-nips-2011,MC-speed-nips-2013}.
The recent LEML method \cite{LEML-ICML-2014} cast the MLML problem into the empirical risk minimization (ERM) framework.
Both MC models and LEML are based on the low rank assumption of the whole label matrix. 
In contrast, the sparse and low rank decomposition is introduced to multi-label learning in a recent work \cite{multilabel-low-rank-sparse-kdd-2016}.
Third, missing labels are treated as latent variables in probabilistic models, including the model based on Bayesian networks 
\cite{multilabel-compressed-sensing-nips-2012,bml-cs-active-kdd-2014} 
and conditional restricted Boltzmann machines (CR-BM).  
Last, Wu et al. \cite{my-icpr-2014} defined three label states, including positive labels $+1$, negative labels $-1$ and missing labels $0$, to avoid the label bias. 
 However, the two solutions proposed in \cite{my-icpr-2014} involves matrix inversion, which limits the scalability to handle larger datasets.
Wu et al. \cite{my-pr-2015} proposed an inductive model based on the framework of regularized logistic regression. It also adopts three label states and a hinge loss function to avoid the label bias. However, the classifier parameters corresponding to each class have to be learned sequentially. 
Furthermore, the computational cost of this method increases significantly with the number of classes, thus, this method becomes prohibitive for very large datasets.

Hierarchical multi-label learning (HML) \cite{ML-reivew-2014} has been applied to problems where the label hierarchy exists, such as image annotation \cite{hierarchy-image-annotation-review-pr-2012}, text classification \cite{hml-text-icml-2005,kernel-hml-text-jmlr-2006} and protein function prediction.
\cite{bi-wei-icml-2011,yu-incomplete-hierarchy-bmc-2015}.
Except for a few cases, most existing HML methods only consider the learning problem of complete hierarchical labels. However, in real problems, the incomplete hierarchical labels commonly occur, such as in image annotation. 
Yu et al. \cite{yu-incomplete-hierarchy-bmc-2015} recently proposed a method to handle the incomplete hierarchical labels. 
However, the semantic hierarchy and the multi-label learning are used separately, such that the semantic hierarchical constraint can not be fully satisfied. 
Deng et al. \cite{deng-eccv-2014} developed a CRF model for object classification. The semantic hierarchical constraint and missing labels are also incorporated into this model. However, a significant difference is that \cite{deng-eccv-2014} focuses on a single object in each instance, while there are multiple object in each instance in our problem.

In the application of image annotation, both missing labels and semantic hierarchy have been explored in many previous works, such as 
\cite{well-multi-label-weak-2010,bucak-multi-incomplete-2011,fasttag-icml-2013,tag-completion-pami-2013,image-tag-missing-cvpr-2013,video-annotation-icm-2008,L1-label-denoising-bmvc-2016,my-aaai-2016-imbalance,li-au-missing-pr-2016} (missing labels) and \cite{hierarchy-image-annotation-review-pr-2012,my-cvpr-2017-dia,my-cvpr-2018-d2ia-gan} (semantic hierarchy).  
However, to the best of our knowledge, no previous work in image annotation has extensively studied missing labels and semantic hierarchy simultaneously. 
Note that the semantic hierarchical constraint used in our model is similar to the ranking constraint 
\cite{ML-calibrated-ranking-2008,bucak-multi-incomplete-2011}
 that is widely used in multi-label ranking models, but there are significant differences. First, the ranking constraint used in these models means the predicted value of the provided positive label should be larger than that of the provided negative label, while the semantic hierarchical constraint involves the ranking of the predicted values between a pair of parent and classes. Besides, the ranking constraint is always incorporated as the loss function, while the semantic hierarchical constraint is formulated as the linear constraint in our model.

\section{Problem and Model}
\label{sec: 3 model}

\subsection{Problem Definition}
\label{sec: 3 subsec problem definition}
Our method takes as input two matrices: a data matrix $\mathbf{X} = [\x_1,\cdots,\x_n] \in \mathbb{R}^{d \times n}$, which aggregates the $d$-dimensional feature vectors of all $n$ (training and testing) instances, and a label matrix $\mathbf{Y} = [\y_1,\cdots,\y_n] \in \{0,\frac{1}{2},+1\}^{m \times n}$, which aggregates the $m$-dimensional label vectors of all instances. That is to say each instance $\x_i$ can take one or more labels from the  $m$ different classes $\{c_1,\ldots,c_m\}$. Its corresponding label vector $\y_i=\mathbf{Y}_{.i}$ determines its membership to each of these classes. For example, if $\mathbf{Y}_{ji} = +1$, then $\x_i$ is a member of $c_j$ and if $\mathbf{Y}_{ji} = 0$, then $\x_i$ is not a member of this class. However, if  $\mathbf{Y}_{ji} = \frac{1}{2}$, then the membership of $\x_i$ to $c_j$ is considered unknown ({\it i.e.}, it has a missing label). Correspondingly, all $m$ labels of each testing instance $\x_k$ are missing, {\it i.e.}, $\y_k=\frac{1}{2}\mathbf{1}$.
The semantic hierarchy is encoded as another matrix: $\mathbf{\Phi} = [\boldsymbol{\phi}_1, \ldots, \boldsymbol{\phi}_{n_e}] \in \mathbb{R}^{m \times n_e}$, with $n_e$ being the number of directed edges.
$\boldsymbol{\phi}_i = [0, \ldots, 1, \ldots, -1, \ldots, 0]^\top$ denotes the index vector of the $i$-th directed edge (see Fig. \ref{fig1: motivation}),
with $\boldsymbol{\phi}_i(i_{parent}) =1$ and $\boldsymbol{\phi}_i(i_{child}) = -1$, while all other entries are 0.

Our goal is to obtain a complete label matrix $\mathbf{Z} \in \{+1,$ $0\}^{m \times n}$ that satisfies the following properties.
\begin{enumerate}
\item 
$\mathbf{Z}$ is consistent with the provided (not missing) labels in $\mathbf{Y}$, {\it i.e.}, $\mathbf{Z}_{ij} = \mathbf{Y}_{ij}$ if $\mathbf{Y}_{ij} \neq \frac{1}{2}$. 
\item 
$\mathbf{Z}$ satisfies the instance-level label similarity. It assumes that $\x_i$ and $\x_j$ have similar features, then their corresponding predicted labels ({\it i.e.}, the $i^{th}$ and $j^{th}$ column of $\mathbf{Z}$) should be similar. 
\item 
$\mathbf{Z}$ follows the class-level label similarity. It assumes that if the co-occurrence between two classes is high, then they will be likely to co-exist at many instances, {\it i.e.}, the corresponding two row vectors of $\mathbf{Z}$ are similar.
\item 
$\mathbf{Z}$ can be decomposed as the sum of a sparse matrix and a low rank matrix, {\it i.e.}, $\Z = \HH_0 + \HH_1$ with $\HH_0$ being low rank and $\HH_1$ being sparse. The rationale of the low rank assumption is that one class could be represented by its related classes. However, due to the existence of tailed labels, the low rank assumption is unlikely to be exactly satisfied. Thus, the sparse matrix is introduced to include the tailed labels, then the remaining label matrix could be low rank. 
\item 
$\mathbf{Z}$ is consistent with the semantic hierarchy $\mathbf{\Phi}$. To enforce this, we ensure that if $c_a$ is the parent of $c_b$, a hard constraint is applied, which guarantees that the score (the presence probability) of $c_a$ should not be smaller than the score of $c_b$. This constraint ensures that the final predicted labels are consistent with the semantic hierarchical constraint.
\end{enumerate} 

Note that both criteria (3) and (4) embed the class-level label dependencies, with (3) being pairwise while (4) being high-order. 
We propose two models to combine (1,2,3,5) and (1,2,4,5) respectively. Note that we can utilize both criteria (3) and (4) to construct a more general model, but to evaluate their different effects, in this manuscript we evaluate two models separately. 
By jointly incorporating all four criteria in model 1 or 2, the label information is propagated from provided labels to the missing labels. 
In what follows, we give a detailed exposition of how these criteria can be mathematically encoded in one unified optimization framework.

\subsection{Label Consistency}
\label{sec: 3 subsec label consistency}

The label consistency of $\mathbf{Z}$ with $\mathbf{Y}$ is enforced using 
\begin{flalign}
\sum_{i,j}^{n, m} \overline{\Y}_{ij} ( \Y_{ij} - \Z_{ij} ) = \text{const} - \tr(\overline{\Y}^\top \Z ),
\label{eq: loss function}
\end{flalign}
where $\text{const} = \tr(\overline{\Y}^\top \Y)$, and 
$\overline{\Y}$ is defined as $\overline{\Y}_{ij} = ( 2 \Y_{ij}-1 ) * \tau_{ij}$, with $\tau_{ij}$ being a penalty factor mismatches between $\mathbf{Y}_{ij}$ and $\mathbf{Z}_{ij}$. We set $\tau_{ij}$ in the following manner. If $\mathbf{Y}_{ij} = 0$, then $\tau_{ij} = r_- >0$, if $\mathbf{Y}_{ij} = +1$, then $\tau_{ij} = r_+ > r_-$, and if $\mathbf{Y}_{ij} = \frac{1}{2}$, then $\tau_{ij} = 0$. 
That is to say a higher penalty is incurred if a ground truth label is $+1$ but is predicted as $0$, as compared to the reverse case. 
This idea reflects the observation that most entries of $\mathbf{Y}$ in many multi-label datasets (with a relatively large number of classes) are $0$ and that $+1$ labels are rare (see the data statistics in Table \ref{table: dataset}). Of course, missing labels are not penalized.

\subsection{Instance-level Label Dependency}
\label{sec: 3 subsec instance level smooth}

Similar to \cite{my-icpr-2014,my-pr-2015}, we incorporate the instance-level label similarity ({\it i.e.}, criteria (2)) using the regularization term in Eq. (\ref{eq: instance-level dependency}).
\begin{flalign}
\vspace{-0.1in}
 \label{eq: instance-level dependency}
\tr(\Z \mathbf{L}_\X \Z^\top) = \hspace{-0.3em}
    \sum_{k,i,j}^{m, n, n} \frac{\W_{\mathbf{X}}(i,j)}{2}  \left[
  \frac{ \Z_{ki}}{\sqrt{\mathbf{d}_{\X}(i)}} -\frac{\Z_{kj}}{\sqrt{\mathbf{d}_{\X}(j)}} \right]^2,
\end{flalign}
where the instance similarity matrix $\mathbf{W}_\X$ is defined as: $\W_{\X}$
 $(i,j)$ $ = \exp{(-\frac{\|\x_{i} - \x_{j}\|^2}{\varepsilon_{i}\varepsilon_{j}}})$. The kernel size $\varepsilon_{i}= \|\x_{i} -\x_{h}\|_2$ and $\x_{h}$ is the $h$-th nearest neighbour of $\x_{i}$ (measured by the Euclidean distance).
Similar to \cite{my-icpr-2014}, we set $h=7$.
The normalization term $\mathbf{d}_{\X}(i)=\sum_{j}^n \W_\X(i,j)$ makes the regularization term invariant to different scaling factors of elements in $\W_{\X}$ \cite{spectral-tutorial-2007}. The normalized Laplacian matrix is $\mathbf{L}_\X=\mathbf{I}-\D^{-\frac{1}{2}}_\X \W_{\X}\D^{-\frac{1}{2}}_\X$ with $\D_\X=\text{diag}\big(\mathbf{d}_{\X}(1),\cdots,$ $\mathbf{d}_{\X}(n)\big)$.

\vspace{4pt}
\subsection{Class-level Label Dependency}\label{sec: 3 subsec class level smooth}

Here, we consider three types of class-level label dependencies, namely class co-occurrence, sparse and low rank decomposition and semantic hierarchy.

\vspace{0.5em} \noindent {\bf Class co-occurrence:}
This dependency is encoded using the regularization term in Eq. (\ref{eq: class-level depedencny}).
\begin{flalign}
\tr(\Z^\top \mathbf{L}_\C \Z)  = \hspace{-0.3em}
 \sum_{k, i,j}^{n, m, m} \frac{\W_{\C}(i,j)}{2}\left[\frac{\Z_{ik}}{\sqrt{\mathbf{d}_{\C}(i)}}-\frac{\Z_{jk}}{\sqrt{\mathbf{d}_{\C}(j)}}\right]^2.
 \label{eq: class-level depedencny}
\end{flalign}
Here, we define the class similarity matrix $\mathbf{W}_\C$ as: $\W_\C(i,j)$ $= \frac{<\Y_{i\cdot},\Y_{j\cdot}>}{\|\Y_{i\cdot}\|\cdot \|\Y_{j\cdot}\|}, ~\forall i \neq j$  and $\W_{\C}(i,i) = 0$. The normalized Laplacian matrix is defined as
$\mathbf{L}_\C=\mathbf{I}-\D^{-\frac{1}{2}}_\C \W_\C \D^{-\frac{1}{2}}_\C$ with   $\D_\C=\text{diag}\big(\mathbf{d}_{\C}(1),\cdots,\mathbf{d}_{\C}(m)\big)$.

\vspace{0.5em} 
\noindent {\bf Sparse and low rank decomposition:}
The sparse and low rank decomposition assumes that the label matrix $\Z$ can be decomposed to the addition of a sparse matrix $\HH_1$ and a low rank matrix $\HH_0$, as follows, 
\begin{eqnarray}
\min_{\HH_0, \HH_1} \gamma_0 \text{rank}(\HH_0) + \gamma_1 \parallel \HH_1 \parallel_{1,1}, ~ \text{s.t.} ~ \Z = \HH_0 + \HH_1.
\label{eq: original sparse and low rank}
\end{eqnarray}
However, it is known that the minimization of $\text{rank}(\HH_0)$ is intractable in general \cite{nuclear-norm-2010}. A widely used solution to minimize its convex approximation \cite{nuclear-norm-low-rank-2002}, {\it i.e.}, the nuclear norm $\parallel \HH_0 \parallel_* = \sum_{i=1}^{\min\{m,n\}}$ $\sigma_i(\HH_0)$, with $\sigma_i(\HH_0)$ being the $i$-th singular value of $\HH_0$. Then the approximation of (\ref{eq: original sparse and low rank}) is formulated as 
\begin{eqnarray}
\min_{\HH_0, \HH_1} \gamma_0 \parallel \HH_0 \parallel_* + \gamma_1 \parallel \HH_1 \parallel_{1,1}, ~ \text{s.t.} ~ \Z = \HH_0 + \HH_1. 
\label{eq: sparse and nuclear norm}
\end{eqnarray}

\vspace{0.5em}  \noindent
{\bf Semantic hierarchical constraint:}
To enforce the semantic hierarchical constraint ({\it i.e.}, criteria (5)), we apply the following constraint: $\mathbf{Z}(i_{parent}, j) \geq \mathbf{Z}(i_{child}, j),~\forall i = 1, \ldots, n_e,  \forall j = 1, \ldots, n$.
The resulting constraints can be aggregated in matrix form,
\begin{eqnarray}
\vspace{-0.1in}
\mathbf{\Phi}^\top \mathbf{Z} \geq 0,
  \label{eq: semantic hierarchy constraint}
\vspace{-0.1in}
\end{eqnarray}
where $\mathbf{\Phi} = [\boldsymbol{\phi}_1, \ldots, \boldsymbol{\phi}_{n_e}] \in \mathbb{R}^{m \times n_e}$. $\boldsymbol{\phi}_i = [0, \ldots, 1, \ldots,$ $-1, \ldots, 0]^\top$ is the indicator vector of the $i$-th directed edge $e_i = (i_{parent} \rightarrow i_{child})$, with $\boldsymbol{\phi}_i(i_{parent}) =1$ and $\boldsymbol{\phi}_i(i_{child})$ $= -1$, with all other entries being 0.

\section{MLML using Mixed Dependency Graph with Co-occurrence (MLMG-CO)}
\label{sec: 3 subsec objective function}

By combining those four properties formulated in Eqs. (\ref{eq: loss function},\ref{eq: instance-level dependency},\ref{eq: class-level depedencny}, \ref{eq: semantic hierarchy constraint}), we construct a mixed dependency graph to connect all label nodes ({\it i.e.}, all entries in $\Z$), referred to as mixed dependency graph with co-occurrence (MG-CO).
Using MG-CO, 
we formulate the MLML problem as a binary matrix optimization problem, where the linear combination of Eqs. (\ref{eq: loss function},\ref{eq: instance-level dependency},\ref{eq: class-level depedencny}) forms the objective and Eq. (\ref{eq: semantic hierarchy constraint}) enforces the semantic hierarchical constraints.
\begin{flalign}
  \min_\Z & \quad -\tr( \overline{\Y}^\top \Z ) + \beta \tr(\Z \mathbf{L}_\X \Z^\top) + \gamma \tr(\Z^\top \mathbf{L}_\C \Z),  \nonumber
  \\
 \text{s.t.}   & \quad \Z\in \{0, 1\}^{m \times n},
  \quad \mathbf{\Phi}^\top \Z  \geq 0,
  \label{eq: obj of matrix based}
\end{flalign}
which is referred to as MLMG-CO. 
The three terms in the objective function correspond to Eqs. (\ref{eq: instance-level dependency},\ref{eq: class-level depedencny},\ref{eq: semantic hierarchy constraint}) respectively. 
Due to the binary constraint on $\mathbf{Z}$, it is difficult to efficiently solve this discrete problem. Thus, we use a conventional \emph{box} relaxation, which relaxes $\mathbf{Z}$ to take on values in $[0,1]^{m \times n}$. Since both $\mathbf{L}_\X$ and $\mathbf{L}_\C$ are positive semi-definite (PSD), it is easy to prove that the relaxed problem of Eq. (\ref{eq: obj of matrix based continuous}) is a convex quadratic problem (QP) with linear matrix constraints (refer to the \textbf{Appendix A} for the detailed proof of the convexity). 
\begin{flalign}
  \min_\Z & \quad -\tr( \overline{\Y}^\top \Z ) + \beta \tr(\Z \mathbf{L}_\X \Z^\top) + \gamma \tr(\Z^\top \mathbf{L}_\C \Z), \nonumber
  \\
 \text{s.t.}   & \quad \Z\in [0, 1]^{m \times n},
  \quad \mathbf{\Phi}^\top \Z  \geq 0.
  \label{eq: obj of matrix based continuous}
\end{flalign}
Due to its convexity and smoothness, the MLMG-CO problem can be efficiently solved by many solvers. In this work, we adopt the alternative direction of method of multipliers (ADMM) \cite{admm-boyd-2011}, which decomposes the optimization problem into several steps that are easy to implement and intuitive to understand.

\subsection{ADMM Algorithm for MLMG-CO}
\label{sec4: subsec ADMM for MLMG-CO}

Following the conventional ADMM framework \cite{admm-boyd-2011}, we firstly formulate the augmented Lagrange function of Problem (\ref{eq: obj of matrix based continuous}), by introducing a non-negative slack variable $\mathbf{Q}\in\mathbb{R}^{n_e \times n}$, 
\begin{flalign}
 & L_\rho(\Z, \Q, \mathbf{\Lambda}) = 
      \beta \tr(\Z \mathbf{L}_\X \Z^\top) + \gamma \tr(\Z^\top \mathbf{L}_\C \Z) -\tr( \overline{\Y}^\top \Z ) 
\nonumber 
 \\
     & + \tr[ \mathbf{\Lambda}^\top (\mathbf{\Phi}^\top \Z - \Q) ] + \frac{\rho}{2} ||(\mathbf{\Phi}^\top \Z - \Q)||_{F}^2, 
       \label{eq: obj of augmented Lagrange}
\end{flalign}
where $\mathbf{Z} \in [0,1]^{m \times n}$ and $\mathbf{Q} \geq 0$. Here, $\mathbf{\Lambda} \mathbf{\Phi} \in \mathbb{R}^{n_e \times n}$ is the Lagrange multiplier (dual variable), $\rho>0$ is a penalty parameter, and $|| \cdot ||_{F}$ denotes the matrix Frobenius norm. 
Then we want to solve the following problem
\begin{flalign}
  \min_{\Z, \Q} \max_{\mathbf{\Lambda}}  L_\rho(\Z, \Q, \mathbf{\Lambda}), 
 ~ \text{s.t.} ~ \mathbf{Z} \in [0,1]^{m \times n}, \mathbf{Q} \geq 0.
\end{flalign}
It can be minimized by alternatively solving the following  sub-problems, with $t$ being the iteration index of the ADMM algorithm.  

\vspace{0.5em}  \noindent
{\bf Sub-problem with respect to $\Z$}: The update of $\Z^{t+1}$ is obtained by the following sub-problem, 
\begin{flalign}
\Z_{t+1} & = \underset{\Z \in [0,1]^{m \times n}}{\arg \min} L_\rho(\Z, \Q_t, \mathbf{\Lambda}_t) 
\label{eq: update of Z in admm}
\\
& = \tr[ \overline{\A}^\top_t \Z ] + \tr[ \Z \overline{\mathbf{B}}_t \Z^\top ] + \tr[ \Z^\top \overline{\C}_t \Z ]  
\nonumber
\end{flalign}
where $\overline{\mathbf{A}}_t = -\overline{\Y} + \mathbf{\Phi} \mathbf{\Lambda}_t - \rho \mathbf{\Phi} \Q_t$, $\overline{\mathbf{B}}_t = \beta \mathbf{L}_\X$ and $\overline{\C}_t = \gamma \mathbf{L}_\C  + \frac{\rho}{2} \mathbf{\Phi} \mathbf{\Phi}^\top$. 
Clearly, $\mathbf{\Phi} \mathbf{\Phi}^\top$ is positive semi-definite (PSD), so $\overline{\C}_t$ is PSD. Considering that $\overline{\mathbf{B}}_t$ is also PSD, thus Problem (\ref{eq: update of Z in admm}) is a convex quadratic programming (QP) problem with box constraints. It can be efficiently solved using projected gradient descent (PGD) with exact line search \cite{boyd-convex-2004}. 

\vspace{0.3em}  \noindent
{\it Projected gradient descent}.
The gradient of the objective function (\ref{eq: update of Z in admm}) with respect to $\Z$ and the step size are computed as
\begin{flalign}
  \nabla \Z_k & = \overline{\mathbf{A}}_k + 2 \Z_k \overline{\mathbf{B}}_k + 2 \overline{\C}_k \Z_k, \\
  \eta_k & = \arg \min_{\eta > 0} L_\rho(\Z_k - \eta \nabla \Z_k, \Q_t, \mathbf{\Lambda}_t)
  \label{eq: step size alpha}
  \\
   & = \frac{ \frac{1}{2} \tr[\overline{\mathbf{A}}_k^\top \nabla \Z_k] + \tr[\Z_k \overline{\mathbf{B}}_k \nabla \Z_k^\top] + \tr[\nabla \Z_k^\top \overline{\C}_k  \Z_k]}{ \tr[\nabla \Z_k\overline{\mathbf{B}}_k \nabla \Z_k^\top] + \tr[\nabla \Z_k^\top \overline{\C}_k  \nabla \Z_k]}, \nonumber
\end{flalign}
where $k$ indicates the iteration index of PGD.
Then $\Z$ is updated as follows:
\begin{flalign}
  \Z_{k+1}  = \min( 1, \max (0, \Z_k - \eta_k \nabla \Z_k) ).
\end{flalign}
The result of the final iteration of PGD will be used as the solution to Problem (\ref{eq: update of Z in admm}), {\it i.e.}, $\Z_{t+1}$.
As Problem (\ref{eq: update of Z in admm}) is convex, PGD is guaranteed to converge to the global optimal solution. However, to reduce the computational cost, we stop this update step only after a few PGD iterations. This heuristic makes the convergence of the overall ADMM much faster, without any considerable effect on performance.

\vspace{0.5em}  \noindent
{\bf Sub-problems with respect to $\Q$ and $\mathbf{\Lambda}$:}
 The updates for $\mathbf{Q}_{t+1}$ and $\mathbf{\Lambda}_{t+1}$ are closed form,
\begin{flalign}
\Q_{t+1} & = \arg \underset{\Q \geq 0 }{\min} \quad L_\rho(\Z_{t+1}, \Q, \mathbf{\Lambda}_t) 
\label{eq: update Q in admm}
\\
& = \max(0, \mathbf{\Phi}^\top \Z_{t+1} + \frac{1}{\rho} \mathbf{\Lambda}^\top_t ) 
\nonumber
 \\
 ~\vspace{.2in}
\mathbf{\Lambda}_{t+1}  & = \mathbf{\Lambda}_t + \rho [\mathbf{\Phi}^\top \Z_{t+1} - \Q_{t+1} ].
~\vspace{.25in}
\label{eq: update lambda in admm}
\end{flalign}

According to the analysis in \cite{admm-proof-2013,admm-proof-2014}, the above ADMM algorithm is guaranteed to converge to the global minimum of Problem (\ref{eq: obj of matrix based continuous}).
Note that if without the semantic hierarchical constraints shown in (\ref{eq: semantic hierarchy constraint}), Problem (\ref{eq: obj of matrix based continuous}) can be more efficiently solved by the PGD algorithm, rather than by ADMM. 

\section{MLML using the Mixed Dependency Graph with Sparse and Low Rank Decomposition (MLMG-SL)}
\label{sec: 5 MLMG-SL}

In this section we propose another formulation of the MLML problem, based on the mixed dependency graph with sparse and low rank decomposition (MG-SL) constructed by Eqs. (\ref{eq: loss function},\ref{eq: instance-level dependency},\ref{eq: sparse and nuclear norm},\ref{eq: semantic hierarchy constraint}), as follows:
\begin{flalign}
 \min_{\Z,\HH_0, \HH_1} & \hspace{-0.2em} \beta\tr(\Z \mathbf{L}_{\X} \Z^\top) + \gamma_0\|\HH_0\|_* + \gamma_1\|\HH_1\|_{1,1} -\tr(\overline{\Y}^\top\Z) 
 \nonumber
\\
\text{s.t.} & ~~ \Z = \HH_0 + \HH_1, \mathbf{\Phi}^\top\Z \geq \mathbf{0}, \Z \in \{0, 1\}^{m\times n},
\label{eq: binary MLMG-SL}
\end{flalign}
which is referred to as MLMG-SL.
Similarly, the binary constraint $\{0, 1\}$ is also relaxed to the box constraint $[0, 1]$, then the relaxed continuous problem becomes 
\begin{flalign}
 \min_{\Z,\HH_0, \HH_1} & -[\alpha\tr(\overline{\Y}^\top\Z) + (1-\alpha)\tr(\overline{\Y}^\top(\HH_0+ \HH_1))] 
\label{eq: continuous MLMG-SL}
\\
 & + \beta\tr(\Z \mathbf{L}_{\X} \Z^\top) + \gamma_0\|\HH_0\|_* + \gamma_1\|\HH_1\|_{1,1}
\nonumber
\\
\text{s.t.} & ~ \Z = \HH_0 + \HH_1, \boldsymbol{\mathbf{\Phi}}^\top\Z \geq \mathbf{0}, \Z \in [0, 1]^{m\times n}. 
\nonumber
\end{flalign}
Note that we have adopted a new loss term in (\ref{eq: continuous MLMG-SL}) by introducing a trade-off parameter $\alpha \in [0,1]$. Due to the constraint $\Z = \HH_0 + \HH_1$, this new loss term is equivalent to the old loss term in (\ref{eq: binary MLMG-SL}). The benefit is the larger flexibility, leading to a more stable convergence in the optimization process. 
As demonstrated in \cite{image-alignment-pami-2012}, both $\|\HH_0\|_*$ and $\|\HH_1\|_{1,1}$ are convex. Considering the convex smoothness term $\tr(\Z \mathbf{L}_{\X} \Z^\top)$ and the linear constraints, the optimization problem in (\ref{eq: continuous MLMG-SL}) is also convex. 
We solve it again using the ADMM algorithm.

\subsection{ADMM Algorithm for MLMG-SL}
\label{sec: 5 subsec ADMM for MLMG-SL} 

The augmented Laplacian function of Problem (\ref{eq: continuous MLMG-SL}) is formulated as follows
\begin{flalign}
& L_{\rho_1, \rho_2}(\Z,\Q, \HH_0,\HH_1, \Lambdaa_1, \Lambdaa_2)  = -[\alpha\tr(\overline{\Y}^\top\Z) + (1-\alpha)
\nonumber
\\ 
& \tr(\overline{\Y}^\top(\HH_0+ \HH_1))] + \beta\tr(\Z \mathbf{L}_{\X}\Z^\top)  + \gamma_0\|\HH_0\|_* +  \gamma_1\|\HH_1\|_{1,1}
\nonumber
\\
&  + \tr[(\Lambdaa_1^\top(\Z-(\HH_0+ \HH_1))] + \tr[\Lambdaa_2^\top(\mathbf{\Phi}^\top\Z-\Q)] 
\nonumber
\\
& + \frac{\rho_1}{2}\|\Z-(\HH_0+ \HH_1)\|_F^2 + \frac{\rho_2}{2}\|\boldsymbol{\mathbf{\Phi}}^\top\Z-\Q\|_F^2,
\label{eq: Laplacian of MLMG-SL}
\end{flalign}
where $\Lambdaa_1 \in \mathbb{R}^{m \times n}$ and $\Lambdaa_2 \in \mathbb{R}^{n_e \times n}$ are two dual variables, and $\rho_1, \rho_2 > 0$ are penalty parameters. 
Then we need to solve the following optimization problem
\begin{flalign}
  \min_{\Z, \Q, \HH_0, \HH_1} \max_{\Lambdaa_1, \Lambdaa_2} &  L_{\rho_1, \rho_2}(\Z,\Q, \HH_0,\HH_1, \Lambdaa_1, \Lambdaa_2), 
 \label{eq: min max Laplacian of MLMG-SL}
\end{flalign}
which can be alternatively solved by optimizing the following sub-problems. 

\vspace{0.3em}  \noindent
{\bf Sub-problem with respect to $\Z$:} 
\begin{flalign}
& \Z_{t+1} =   \underset{\Z\in[0,1]^{m\times n}}{\arg\min} L_{\rho_1, \rho_2}(\Z,\Q_t,\Lambdaa_t,\HH_{0(t)},\HH_{1(t)}) 
\label{eq: update of Z in admm for MLMG-SL}
\\
& =  
- \alpha\tr(\overline{\Y}^\top\Z) + \beta\tr(\Z \mathbf{L}_{\X}\Z^\top)   
+ \tr(\Lambdaa_{1(t)}^\top\Z) + \frac{\rho_1}{2}\|\Z
\nonumber
\\
& - \HH_{0(t)} - \HH_{1(t)})\|_F^2 + \tr(\Lambdaa_{2(t)}^\top \mathbf{\Phi}^\top\Z) + \frac{\rho_2}{2}\|\mathbf{\Phi}^\top\Z-\Q_t\|_F^2 
\nonumber
\\
& = \underset{\Z\in[0,1]^{m\times n}}{\arg\min} \tr[(-\alpha\overline{\Y} + \Lambdaa_{1(t)} - \rho_1(\HH_{0(t)}+ \HH_{1(t)})
\nonumber
\\
& - \rho_2 \mathbf{\Phi}\Q_t + \mathbf{\Phi}\Lambdaa_{2(t)})^\top\Z] + \beta\tr(\Z \mathbf{L}_{\X}\Z^\top) + \tr[\Z^\top(\frac{\rho_1}{2}
\nonumber
\\
& + \frac{\rho_2}{2}\mathbf{\Phi}\mathbf{\Phi}^\top)\Z]
\nonumber
\end{flalign}
Similar to the sub-problem (\ref{eq: update of Z in admm}), it is not hard to see that (\ref{eq: update of Z in admm for MLMG-SL}) is convex, which can also be efficiently solved by the PGD algorithm with line search.

\vspace{0.3em}  \noindent
{\bf Sub-problem with respect to $\HH_0$:}  
\begin{flalign}
& \HH_{0(t+1)} =  \underset{\HH_{0}}{\arg\min} L_{\rho_1, \rho_2}(\Z_{t+1},\Q_t,\Lambdaa_t,\HH_{0},\HH_{1(t)}) 
\label{eq: update of H0 in admm for MLMG-SL}
\\
& = 
 - \tr(\Lambdaa_{1(t)}^\top\HH_{0(t)}) + \gamma_0\|\HH_{0(t)}\|_* + \frac{\rho_1}{2}\|\Z_{t+1} - \HH_{0(t)} 
\nonumber 
\\
&  - \HH_{1(t)}\|_F^2 + (1-\alpha)\tr(\overline{\Y}^\top\HH_{0(t)})
\nonumber
 \\
& =\underset{\HH_{0}}{\arg\min} \frac{\gamma_0}{\rho_1}\|\HH_{0(t)}\|_* +\frac{1}{2}\|\HH_{0(t)}-( \Z_{t+1} - \HH_{1(t)} + \mathbf{E}) \|_F^2
\nonumber
\\
&   
= \mathcal{D}_{\frac{\gamma_0}{\rho_1}}(\Z_{t+1} - \HH_{1(t)} + \mathbf{E}), 
\nonumber
\end{flalign}
where we define $\mathbf{E} = \frac{\Lambdaa_{1(t)}+ (1-\alpha)\overline{\Y}}{\rho_1}$ to save space. 
$\mathcal{D}_{\lambda}(\A)$ $= \U_{\A} \mathcal{S}_{\lambda}(\mathbf{\Sigma}_{\A}) \V_{\A}^\top$ denotes the singular value soft-thresholding operator \cite{image-alignment-pami-2012}, utilizing the soft-thresholding operator 
$\mathcal{S}_{\lambda}(\B_{ij})$ $= \text{sign}(\B_{ij})\max(0, |\B_{ij}| - \lambda)$ and the SVD decomposition $\A = \U_{\A} \mathbf{\Sigma}_{\A} \V_{\A}^\top$.

\vspace{0.3em}  \noindent
{\bf Sub-problem with respect to $\HH_1$:}  
\begin{flalign}
& \HH_{1(t+1)}  =  \underset{\HH_{1}}{\arg\min} L_\rho(\Z_{t+1},\Q_t,\Lambdaa_t,\HH_{0(t+1)},\HH_{1}) 
\label{eq: update of H1 in admm for MLMG-SL}
 \\
 & = (\alpha -1) \tr(\overline{\Y}^\top \HH_1) +
  \gamma_1 \| \HH_1 \|_{1,1} - \tr(\Lambdaa_{1(t)}^\top \HH_1) 
\nonumber 
\\
& + \frac{\rho_1}{2}\|\Z_{t+1} -(\HH_{0(t+1)}+ \HH_1)\|_F^2
\nonumber 
\\
& =\underset{\HH_{1}}{\arg\min} \frac{\gamma_1}{\rho_1} \| \HH_1 \|_{1,1} + \frac{1}{2} \| \HH_1 - (\Z_{t+1} - \HH_{0(t+1)} + \mathbf{E} )\|_F^2
\nonumber 
\\
  & =\mathcal{S}_{\frac{\gamma_1}{\rho_1}}(\Z_{t+1} - \HH_{0(t+1)} +  \mathbf{E})
\nonumber
\end{flalign}
where the soft-thresholding operator $\mathcal{S}_{\frac{\gamma_1}{\rho_1}}(\cdot)$ and $\mathbf{E}$ are defined as above.

\vspace{0.3em}  \noindent
{\bf Sub-problems with respect to $\Q$, $\Lambdaa_{1}$ and $\Lambdaa_{2}$:}  
\begin{flalign}
\Q_{t+1} &= \max(\boldsymbol{0}, \mathbf{\Phi}^\top\Z_{t+1} + \frac{1}{\rho_2}\Lambdaa_{2(t)}^\top) 
\label{eq: update of Q in admm for MLMG-SL}
\\
\Lambdaa_{1(t+1)} & = \Lambdaa_{1(t)} + \rho_1(\Z_{t+1} - \HH_{0(t+1)} - \HH_{1(t+1)})
\label{eq: update of lambda_1 in admm for MLMG-SL}
\\ 
\Lambdaa_{2(t+1)} & = \Lambdaa_{2(t)} + \rho_2(\mathbf{\Phi}^\top\Z_{t+1} - \Q_{t+1})
\label{eq: update of lambda_2 in admm for MLMG-SL}
\end{flalign}

In terms of the convergence, as demonstrated in \cite{ADMM-multiblock-not-convergent-2016}, the ADMM algorithm for multi-block (more than 2 blocks) convex optimization is not necessarily convergent. 
Some further assumptions about the objective function or the parameters $\rho_1, \rho_2$ should be added to guarantee the convergence.  
For example, a recent work \cite{ADMM-three-block-convergence-2016} has proved that if the variable sequence generated by the above ADMM algorithm is assumed to satisfy the sub-strong monotonicity, and the parameters $\rho_1, \rho_2$ are set in a bounded range, then the algorithm will converge to a KKT solution. Please refer to \cite{ADMM-three-block-convergence-2016} for more details.

\section{Experiments} \label{sec: 5 experiments}
In this section, 
we evaluate the proposed method and the state-of-the-art methods on four benchmark datasets in image annotation and video annotation.

\begin{figure*}[!htb]
\centering
\includegraphics[width=\textwidth,height=1.9in]{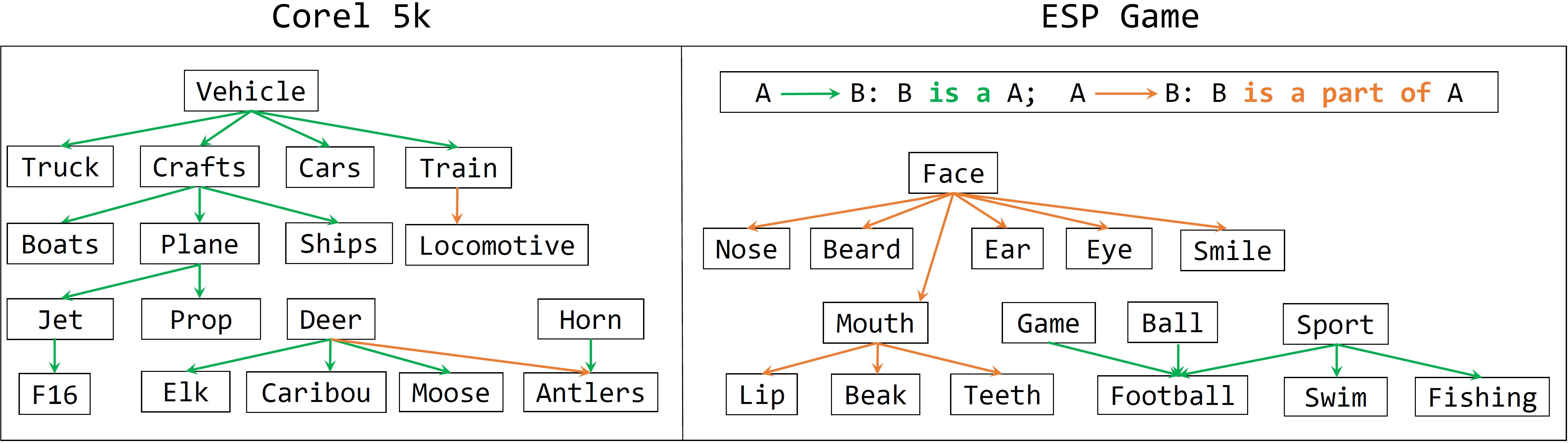}
\caption{A part of semantic hierarchies of Corel 5k and ESP Game, respectively.}
\label{fig: semantic hierarchy corel5k and espgame}
\end{figure*}

\begin{table}[t]
\caption{ Details of the semantic hierarchies for four datasets that we augmented. Column notations C1 to C6 respectively indicate the number of: nodes, edges, root nodes, leaf nodes, singleton nodes and depth.}
\label{table: Semantic hierarchy}
\vspace{-0.05in}
\begin{center}
\scalebox{1.02}{
\begin{tabular}{| p{.12\textwidth} | p{.03\textwidth} p{.03\textwidth}
p{.03\textwidth} p{.03\textwidth} p{.03\textwidth} p{.03\textwidth}|}
\hline
dataset & C1 & C2 & C3 & C4 & C5 & C6
\\
\hline \hline
Corel 5k \cite{corel5k-eccv-2002}
& 260 & 138 &  37 & 98 & 99 & 5
\\
ESP Game \cite{esp-game-2004}
& 268 & 129 & 41 & 92 & 120 & 4
\\
IAPRTC-12 \cite{iaprtc-12-data-2006}
&  291 & 179 & 36 & 132 & 98 & 4
\\
MediaMill \cite{mediamill-data-2006}
& 101 & 63 & 14 & 52 & 30 & 3
\\
\hline
\end{tabular}
}
\end{center}
\end{table}

\subsection{Experimental Setup}
\label{sec: 5 subsec experimental setup}

\noindent\textbf{Datasets}. Four benchmark multi-label datasets are used in our experiments, including Corel 5k \cite{corel5k-eccv-2002}, ESP Game \cite{esp-game-2004}, IAPRTC-12 \cite{iaprtc-12-data-2006}, and MediaMill \cite{mediamill-data-2006}. These datasets are chosen because they are representative and popular benchmarks for comparative analysis among MLML methods. 
The features and labels of the first three image datasets are downloaded from the seminal work \cite{multilabel-dataset-image-iccv-2009} \footnote{http://lear.inrialpes.fr/people/guillaumin/data.php}. 
Each image in these datasets is described by the dense SIFT features and is represented by a 1000-dimensional vector.
Moreover, the original images of ESP Game and IAPRTC-12 are also available. Thus we can extract other features. It is known that the deep feature extracted from CNNs shows surprising performance in many image-based tasks. Thus we also adopt the CNN features in our experiments for this two datasets. Specifically, the output of the relu7 layer of the pre-trained VGG-F\footnote{http://www.vlfeat.org/matconvnet/pretrained/} \cite{vggf-bmvc-2014} model is extracted as the feature vector of 4096 dimensions. 
The features and labels of the video dataset MediaMill are downloaded from the `Mulan' website \footnote{\text{http://mulan.sourceforge.net/datasets-mlc.html}}.

\vspace{4pt}\noindent \textbf{Semantic hierarchies}. We build semantic hierarchies for each dataset based on WordNet \cite{wordnet-1998}. Specifically, for each dataset, we search for each class in Wordnet and extract one or 
more directed paths ({\it i.e.}, a long sequence of directed edges from parent class to child class).
In each path, we identify the nearest upstream class that is also in the label vocabulary ({\it i.e.}, the set $\{c_1,\cdots,c_m\}$ of all classes of the dataset of interest) as the parent class. This procedure is repeated for all $m$ classes in this dataset to form the semantic hierarchy matrix $\mathbf{\Phi}$. In the same manner, we build the hierarchy for each of the four datasets. 
Similar to \cite{hierarchy-image-annotation-review-pr-2012}, we also consider two types of semantic dependency: `is a' and `is a part of'. For example, a part of the semantic hierarchy of Corel 5k and ESP Game is shown in Fig. \ref{fig: semantic hierarchy corel5k and espgame}.
Note that not all `is a part of' dependencies are included in the semantic hierarchy, to ensure the corresponding semantic hierarchical constraint to be correct. For example, ``tree is a part of forest'', but when `tree' exists in one image, `forest' doesn’t always exist, so we abandon it.
 A summary of these semantic hierarchies\footnote{The complete semantic hierarchies and the complete label matrices of all four datasets can be downloaded from ``https://sites.google.com/site/baoyuanwu2015/".}
 is presented in Table \ref{table: Semantic hierarchy}.

Note that in aforementioned datasets, the provided ground-truth label matrices do not fully satisfy the semantic hierarchical constraints. In other words, some instances are labelled with a child class but not with the corresponding parent class. Therefore, we augment the label matrix according to the semantic hierarchy for each dataset. The semantically enhanced comprehensive ground-truth label matrix is referred to as ``complete", while the originally provided label matrix as ``original". The basic statistics of both the complete and original label matrices are summarized in Table \ref{table: dataset}.

\vspace{4pt}
\renewcommand{\arraystretch}{1.08}
\begin{table*}[!tbh]
\caption{ Data statistics of features and label matrices of four benchmark datasets. The column indexes C1 to C5 respectively indicate: the dimension of traditional features, the dimension of CNN features, the average positive classes of each instance, the average positive instances of each class and the positive label proportion in the whole training label matrix.}
\label{table: dataset}
\begin{small}
\begin{center}
\scalebox{0.87}{
\begin{tabular}{| p{.14\textwidth}  | p{.21\textwidth} p{.05\textwidth} p{.05\textwidth} p{.05\textwidth} p{.12\textwidth} | p{.11\textwidth} | p{.055\textwidth}  p{.055\textwidth} p{.055\textwidth} |}
\hline
dataset &  \scalebox{0.93}{\# instances (training, test)} & \# class & C1 & C2 & $k_X$, $k_C, \frac{r_+}{r_-}$ & label matrix & C3 & C4 & C5 \\
\hline \hline
\multirow{2}{*}{Corel 5k \cite{corel5k-eccv-2002} }
&  \multirow{2}{*}{4999 = 4500 + 499} & \multirow{2}{*}{260} & \multirow{2}{*}{1000} & \multirow{2}{*}{\text{N/A}} & \multirow{2}{*}{20, 10, 100} & original & 3.40 & 65.30 & 1.31\%
\\ \cline{7-10}
  &  &  &  & & & complete & 4.84 & 93.06 & 1.86\%
\\
\hline 
\multirow{2}{*}{MediaMill \cite{mediamill-data-2006} }  & \multirow{2}{*}{43907 = 30993 + 12914} & \multirow{2}{*}{101} & \multirow{2}{*}{120} & \multirow{2}{*}{\text{N/A}} & \multirow{2}{*}{20, 10, 100} & original & 4.38 & 1902 & 4.33\%
\\ \cline{7-10}
  & &  &  & &  & complete & 6.17 & 2680 & 6.10\%
  \\
\hline
\multirow{2}{*}{ESP Game \cite{esp-game-2004}}
& \multirow{2}{*}{20770 = 18689 + 2081} & \multirow{2}{*}{268} & \multirow{2}{*}{1000} & \multirow{2}{*}{4096} & \multirow{2}{*}{20, 10, 100} & original & 4.69 & 363.2 & 1.75\%
\\ \cline{7-10}
   & &  &  & & & complete & 7.27 & 563.6 & 2.71\%
  \\
\hline 
\multirow{2}{*}{IAPRTC-12 \cite{iaprtc-12-data-2006} }
&  \multirow{2}{*}{19627 = 17665 + 1962} & \multirow{2}{*}{291} & \multirow{2}{*}{1000} & \multirow{2}{*}{4096} & \multirow{2}{*}{20, 10, 100} & original & 5.72 & 385.71 & 1.97\%
\\ \cline{7-10}
   & &  &  & & & complete & 9.88 & 666.3 & 3.39\%
  \\
\hline 
\end{tabular}
}
\end{center}
\end{small}
\end{table*}

\vspace{4pt}\noindent\textbf{Methods for comparison}. 
In our methods, there are two places we use semantic hierarchies. One is to fill in the original initial label matrix $\Y$, {\it i.e.}, if $\Y(i, j) = 1$, then $\Y(pa(i), j)$ is set to $1$. $pa(i)$ denotes the ancestor classes of class $i$ in the semantic hierarchy. 
If we do this filling in $\Y$, then it is referred to as {\it filling initial label matrix}, otherwise {\it not-filling initial label matrix}. 
The other place is to construct the constraint matrix $\mathbf{\Phi}$ (see Eq. (\ref{eq: semantic hierarchy constraint})). To evaluate the influences of this two usages, we compare different variants of our methods, as shown in Table \ref{table: different algorithm names}. 
Several state-of-the-art multi-label methods that can also handle missing labels are used for comparison, including 
MC-Pos \cite{MC-Pos-nips-2011}, FastTag \cite{fasttag-icml-2013}, MLML-exact and MLML-appro \cite{my-icpr-2014}, as well as LEML \cite{LEML-ICML-2014}. 
FastTag is specially developed for image annotation, while other methods are general machine learning methods.
Also, a state-of-the-art method in hierarchical multi-label learning, called CSSAG \cite{bi-wei-icml-2011}, is also evaluated. CSSAG is a decoding method based on the predicted continuous label matrix of one another algorithm, {\it i.e.}, the kernel dependency estimation (KDE) algorithm \cite{kde-nips-2002}. However, the KDE algorithm doesn't work in the case of missing labels.
To make a fair comparison between CSSAG and our proposed methods,
the predicted label matrix of MLMG-CO is used as the input of CSSAG. The results are obtained with publicly available MATLAB source code of these methods provided by their authors.
Note that in our previous work \cite{my-iccv-2015}, MLR-GL \cite{bucak-multi-incomplete-2011} and the binary SVM were also compared, but here we choose to remove the comparisons with them, due to their much higher costs on both computation and memory than other compared methods.

\begin{table*}[phtb]
\caption{ Different algorithm names of variants of our methods. See ``Methods for comparison" in Section \ref{sec: 5 subsec experimental setup} for details. }
\label{table: different algorithm names}
\vspace{-0.05in}
\begin{small}
\begin{center}
\scalebox{0.986}{
\begin{tabular}{| p{.17\textwidth} | p{.15\textwidth} p{.205\textwidth} |
p{.15\textwidth} p{.205\textwidth} |}
\hline
model $\rightarrow$ & \multicolumn{2}{c}{MLMG-CO} \vline & \multicolumn{2}{c}{MLMG-SL} \vline \\
\hline \hline
constraint $\downarrow$, \quad initial $\rightarrow$ & not-filling & filling & not-filling & filling 
\\
\hline
without SH constraint & MLMG-CO & \scalebox{0.9}{MLMG-CO + filling}  & MLMG-SL  & \scalebox{0.9}{MLMG-SL + filling}
\\
with SH constraint & \scalebox{0.9}{MLMG-CO + constraint} & \scalebox{0.9}{MLMG-CO + filling + constraint}  & \scalebox{0.9}{MLMG-SL + constraint} & \scalebox{0.9}{MLMG-SL + constraint + filling}
\\
\hline
\end{tabular}
}
\end{center}
\end{small}
\end{table*}

\vspace{4pt}\noindent\textbf{Evaluation metrics}. 
Average precision (AP) \cite{multilabel-evaluation-tkdd-2010} is adopted to measure the ranking performance of the predicted labels of each instance, {\it i.e.}, the ranking performance of each column vector in the continuous label matrix $\mathbf{Z}$. 
Mean average precision (mAP) \cite{information-retrieval-2008} is also adopted to evaluate the performance of the tag-based image retrieval, {\it i.e.}, the ranking performance of each row vector in $\mathbf{Z}$. 
To quantify the degree to which the semantic hierarchical constraints are violated, we adopt a simplified hierarchical Hamming loss, similar to \cite{hml-text-icml-2005},
~\vspace{-.1in}
\begin{flalign}
 \ell_H^k(\hat{\Z}_k, \Y_C) = & 
 \frac{1}{nm}\sum_{i,j}^{n,m}\mathbb{I}\big[ (\Y_C(pa(i), j) = 0) \wedge 
 \label{eq: HL_0}
 \\
& 
 (\hat{\Z}_k(pa(i), j)=0) \wedge (\hat{\Z}_k(i,j)=1) \big],
  \nonumber 
\end{flalign}
where $\hat{\Z}_k$ denotes the discrete label matrix generated by setting the top-$k$ labels in the continuous label vector of each instance as $+1$, while all others as $0$. $\Y_C$ denotes the complete ground-truth label matrix. 
$\wedge$ indicates the logical AND operator. 
$\mathbb{I}(a)$ denotes the indicator function: if $a$ is {\it true}, then $\mathbb{I}(a)=1$, otherwise $\mathbb{I}(a)=0$.
The above equation calculates the case that in the ground-truth $\Y_C(pa(i), j) = 0$, if the predicted label of the parent class is correct ({\it i.e.}, $\hat{\Z}_k(pa(i), j)=0$) but the label of the child class is incorrect ({\it i.e.}, $\hat{\Z}_k(i, j)=1$). This case indicates the violation of semantic hierarchical constraints.
Then we define an {\it average hierarchical loss} (AHL) as $\overline{\ell}_H = \frac{1}{|\mathcal{S}|} \sum_{k \in \mathcal{S}} \ell_H^k$. In experiments we set $\mathcal{S} = \{5, 10, 20, 50, 80\}$ on MediaMill, while $\mathcal{S} = \{5, 10, 20, 50, 100, 150\}$ on other datasets.

\begin{figure*}[!tbh]
\centering
\includegraphics[width=\textwidth, height=2.5in]{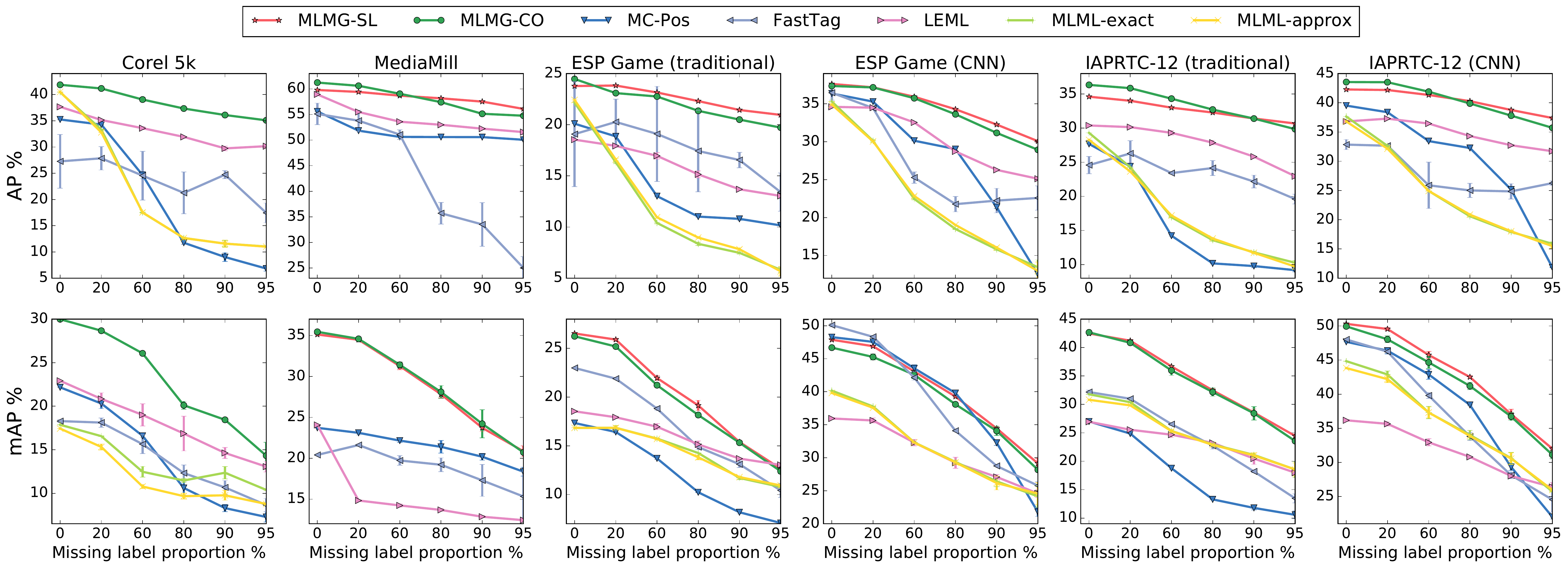}
\caption{Average precision (\textbf{top}) and mAP (\textbf{bottom}) results of four benchmark datasets for methods with the original initial label matrix. The bar on each point indicates the corresponding standard deviation. Figure better viewed on screen.}
\label{fig: AP and MAP results with notfill initial}
\end{figure*}

\begin{figure*}[!tbh]
\centering
\includegraphics[width=\textwidth, height=2.5in]{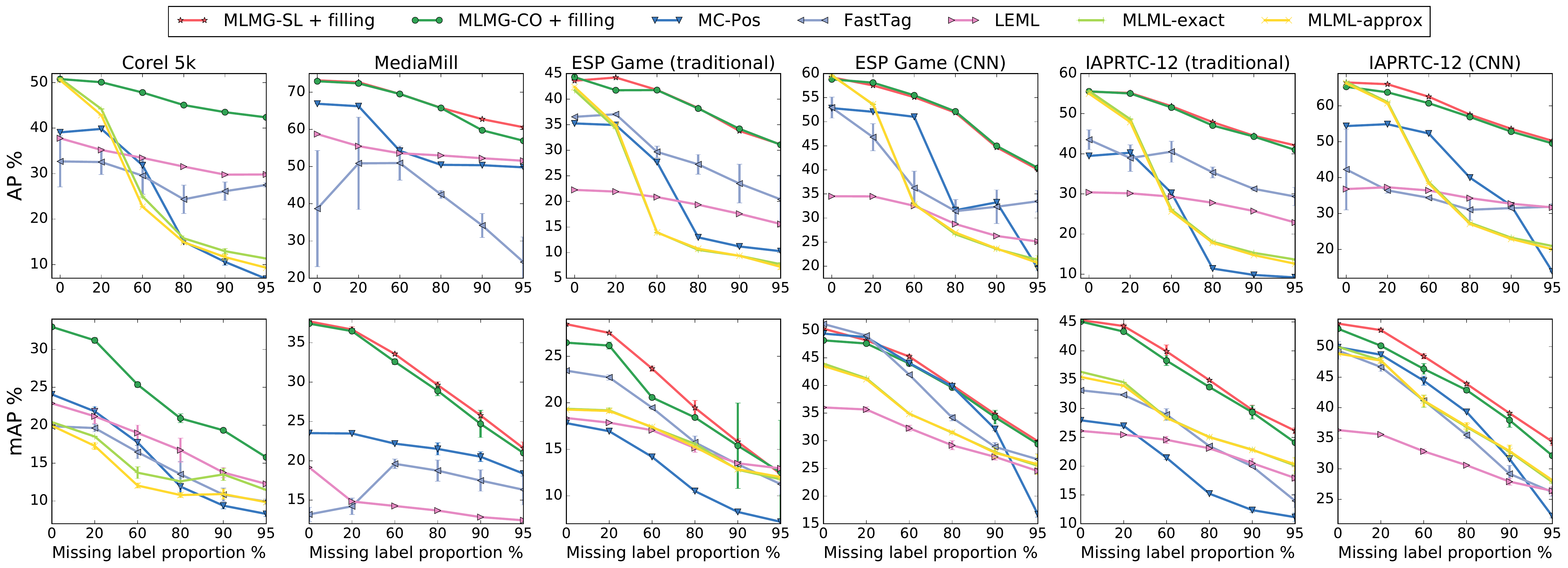}
\caption{Average precision (\textbf{top}) and mAP (\textbf{bottom}) results of four benchmark datasets for methods with the semantically filled-in initial label matrix. The bar on each point indicates the corresponding standard deviation. Figure better viewed on screen.}
\label{fig: AP and MAP results with fillin initial}
\end{figure*}

\vspace{4pt}\noindent\textbf{Other settings}. 
To simulate different scenarios with missing labels, we create training datasets with varying portions of missing labels, ranging from $0\%$ to $95\%$. 
Given a missing label proportion $\tau \in [0, 1)$, firstly we randomly sample rounding$(mn_{tr}\tau)$ entries in the training label matrix, with $n_{tr}$ being the number of training instances.
Then, for every sampled entry, we check whether it corresponds to the leaf or singleton classes in the constructed semantic hierarchies introduced above: if yes, choose it as a missing label, otherwise keep its original value in the training label matrix. 
Consequently, the number of missing labels is smaller than rounding$(mn_{tr}\tau)$. 
The reason of this setting is that if missing labels could be generated on root and intermediate classes, many of them can be directly inferred as positive labels using the semantic hierarchical constraint. Specifically, given one missing label generated on root or intermediate classes, if any one of its descendant classes is positive, then this missing label could be easily corrected to positive. 
Note that this setting is more favourable to other compared methods that don't utilize the semantic hierarchical constraint. 
We repeat the above process 5 times to obtain different missing labels. In all cases, the experimental results of testing data are computed based on the complete label matrix. The reported results are summarized as the mean and standard deviation over all the runs.
The trade-off parameters of MLMG-CO ($\beta$ and $\gamma$) and MLMG-SL ($\alpha$, $\beta$, $\gamma_0$ and $\gamma_1$) are tuned by cross-validation. 
Specifically, for MLMG-CO, we set the tuning ranges as $\beta \in \{0.1, 1, 5, 10, 50\}$, and $\gamma \in \{0, 0.01, 0.1, 1, 10\}$; for MLMG-SL, they are $\alpha \in \{0.1, 0.5, 0.9, 1\}$, $\beta \in \{0.1, 1, 5, 10, 50\}$, $\gamma_0 \in \{0.0001,$ $0.001, 0.01, 1, 10\}$ and $\gamma_1 \in \{0.1, 1, 10, 100, 1000\}$.
$\W_\X$ and $\W_\C$ are defined as sparse matrices. The numbers of neighbors of each instance/class $k_\X$ and $k_\C$ are set as $20$ and $10$, respectively.

\vspace{2pt}\noindent{\bf An acceleration heuristic}.
The computation of the step size $\eta_k$ (see Eq. (\ref{eq: step size alpha})) in MLMG-CO takes about $50\%$ of the running time in each iteration. However, we observe that the step size in consecutive iterations tend to be very close. Thus, we only compute the step size $\eta_k$ once in every 5 iterations, while other consecutive step sizes are derived by multiplying a damping factor ($0.9$ in our experiments) with that of their last iterations. Compared to the case where the step size is computed exactly in each iteration, the runtime is significantly reduced to about $40\%$ (this value depends on $(m,n)$) with a negligible effect in prediction performance.

\subsection{Results without Semantic Hierarchical Constraints}
\label{sec: 5 subsec results without semantic hierarchy}

Figs. \ref{fig: AP and MAP results with notfill initial} and \ref{fig: AP and MAP results with fillin initial} present AP and mAP results when the semantic hierarchy is not used as constraint, {\it i.e.}, $\mathbf{\Phi}=\mathbf{0}$. 
In this case, the inequality constraints (see (\ref{eq: semantic hierarchy constraint}))  in ML-MG are degenerate. Then the proposed model MLMG-CO is a convex QP with box constraints that is solvable using the PGD algorithm, which is more efficient than the ADMM algorithm. 
The semantic hierarchy is only used to fill in the missed ancestor labels in the initial label matrix $\Y$. 
We report both results of using the original initial label matrix and using the semantically filled-in initial label matrix, as shown in Figs. \ref{fig: AP and MAP results with notfill initial} and \ref{fig: AP and MAP results with fillin initial},  respectively. 
With the same initial label matrix and without constraints, it ensures the fair comparison among the formulations in different models for the MLML problem.

As shown in Fig. \ref{fig: AP and MAP results with notfill initial}, 
both MLMG-CO and MLMG-SL consistently outperform other MLML methods, even without using the semantic hierarchy information.
The improvement margin over the most competitive method on the six datasets is at least $5\%$ (AP) or $3\%$ (mAP). Compared with MLML-exact and MLML-approx, MLMG-CO shows significant improvement, especially when large proportions of missing labels exist. There are two main reasons. 
Firstly, there are many noisy negative labels in the original training label matrix, {\it i.e.}, some positive labels 1 are incorrectly set to 0. Since a larger penalty is incurred when misclassifying a positive label in MLMG-CO, the influence of noisy negative labels can be alleviated. However, this is not the case for both MLML-exact and MLML-approx. 
Secondly, MLMG-CO does not give any bias to missing labels. In contrast, missing labels are encouraged to be intermediate values between negative and positive labels in MLML-exact and MLML-approx, which brings in label bias. This is why their performance decreases significantly as the missing proportion increases.

\begin{figure*}[tbph]
\centering
\includegraphics[width=\textwidth, height=3.3in]{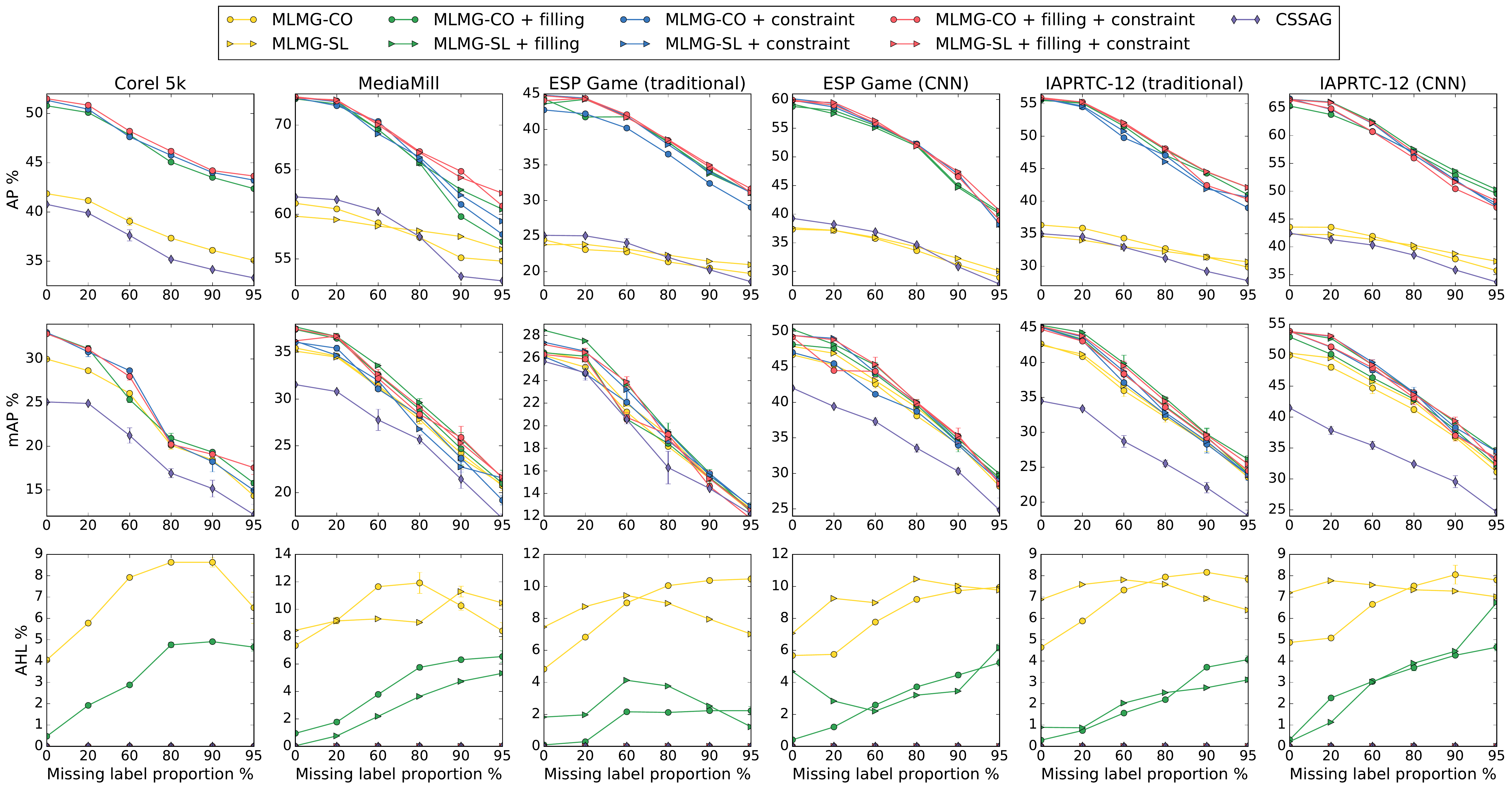}
\caption{Average precision (AP) and average hierarchical loss (AHL) results of our methods and CSSAG.}
\label{fig: AP, mAP and AHL results of all cases}
~\vspace{-.2in}
\end{figure*}

In terms of the comparison between MLMG-CO and MLMG-SL, their performance is similar at most cases. However, we observe that when the missing label proportion is small, MLMG-CO is slightly better than MLMG-SL; as the missing label proportion increases, MLMG-SL shows better performance than MLMG-CO. Specifically, at the case of $95\%$ missing labels, the relative improvements at AP values of MLMG-SL over MLMG-CO are $2.43\%, 6.29\%, 3.94\%$,  $2.78\%, 4.7\%$ on MediaMill, ESP Game (traditional), ESP Game (CNN), IAPRTC-12 (traditional) and IAPRTC-12 ( CNN), respectively; while the relative improvements at mAP values are $0.68\%$, $1.29\%, 3.41\%,$ $3.82\%,$ $2.78\%$,  accordingly. 
It is consistent with the expectations of different assumptions used in MLMG-CO and MLMG-SL. As the class-level smoothness used in MLMG-CO is derived from the initial label matrix, when massive missing labels exist, the obtained smoothness is likely to be inaccurate; in contrast, the sparse and low rank decomposition used in MLMG-SL is independent of the initial label matrix, thus it will not be influenced by the increased missing labels. 
Note that on Corel 5k, due to the extremely sparse positive labels in the label matrix (see the positive proportions in Table \ref{table: dataset}), the SVD step in MLMG-SL algorithm cannot lead to the valid solution, thus the results of MLMG-SL are not reported. 
Besides, the high memory requirements of MLML-exact and MLML-approx preclude running them on MediaMill data. 

The results of using the filled-in initial label matrix are shown in Fig. \ref{fig: AP and MAP results with fillin initial}. 
Similarly, both MLMG-CO and MLMG-SL show much better performance than other compared methods.
At the case of $95\%$ missing labels, the relative improvements at AP values of MLMG-SL over MLMG-CO are $6.38\%, 0.23\%, -0.87\%,$ $2.76\%, 1.43\%$ on MediaMill, ESP Game (traditional), ESP Game (CNN), IAPRTC-12 (traditional) and IAPRTC-12 (CNN), respectively; while the relative improvements at mAP values are $2.91\%,-1.44\%, 1.73$ 
 $\%, 8.51\%, 7.12\%$, accordingly. 

Comparing Figs. \ref{fig: AP and MAP results with notfill initial} and \ref{fig: AP and MAP results with fillin initial}, it is easy to see that the performance of most methods are significantly improved of using the filled-in initial label matrix over using the original initial label matrix. The main reason is that the performance of any models will be significantly influenced by the noisy labels ({\it i.e.}, the ground-truth positive labels are incorrectly set as negative labels in the original initial label matrix). It verifies the contribution of the augmented ground-truth label matrix using our constructed semantic hierarchies.

\subsection{Results with Semantic Hierarchical Constraints} \label{sec: 5 subsec results with semantic hierarchy}
The results of utilizing the semantic hierarchy are shown in Fig. \ref{fig: AP, mAP and AHL results of all cases}. 
To highlight the influence of semantic hierarchical constraints, here we again report the results of MLMG-CO, MLMG-CO + filling, MLMG-SL and MLMG-SL + filling, which have been presented in Section \ref{sec: 5 subsec results without semantic hierarchy}. 

\vspace{4pt}\noindent
\textbf{Comparison among four variants of MLMG-CO}. 
In Fig. \ref{fig: AP, mAP and AHL results of all cases}, the results of four variants of MLMG-CO are denoted using the lines with the $\circ$ mark, but with different colors. 
The results of MLMG-CO are much inferior to those of the other three variants, because the semantic hierarchy is neither used in the initial label matrix, nor as constraints during the optimization. This demonstrates the importance of the semantic hierarchy. 
MLMG-CO + filling and MLMG-CO + constraint show the similar performance evaluated by AP and mAP. For MLMG-CO + constraint, although there are many noisy labels in the initial label matrix, the constraint during the optimization can correct the noisy labels to a large extent, to achieve the similar ranking (AP and mAP) performance with MLMG-CO + filling. However, the AHL values of MLMG-CO + constraint are always 0, while those of MLMG-CO + filling are always positive. This tells that the tag ranking list for each instance produced by MLMG-CO + constraint is semantically consistent, while that produced by MLMG-CO + filling is partially inconsistent with the semantic hierarchical constraint, {\it i.e.}, some children tags are ranked higher than their ancestor tags. 
These two points verify the efficacy of embedding the semantic hierarchy as the linear constraint. 

\renewcommand{\arraystretch}{1.05}
\begin{table*}[phtb]
\caption{ Evaluation of our proposed methods in the semi-supervised multi-label setting. `provided' indicates the image subset of the fully labelled images in the original training set; `val' represents the image subset of the unlabelled images in the original training set; 'testing' denotes the testing images set, where the images are also unlabelled.}
\label{table: ssl of MLMG-CO and MLMG-SL}
\vspace{-0.01in}
\begin{small}
\begin{center}
\scalebox{0.805}{
\begin{tabular}{| p{.06\textwidth} | p{.085\textwidth} | p{.045\textwidth} p{.045\textwidth} p{.045\textwidth} | p{.045\textwidth} p{.045\textwidth} p{.045\textwidth} | p{.045\textwidth} p{.045\textwidth} p{.045\textwidth} | p{.045\textwidth} p{.045\textwidth} p{.045\textwidth} | p{.045\textwidth} p{.045\textwidth} p{.045\textwidth} |}
\hline
\multirow{2}{*}{evaluation} & \multirow{2}{*}{method} & \multicolumn{3}{c}{MediaMill} \vline & \multicolumn{3}{c}{ESP Game (traditional)} \vline & \multicolumn{3}{c}{ESP Game (CNN)}  \vline  
& \multicolumn{3}{c}{IAPRTC-12 (traditional)} \vline & \multicolumn{3}{c}{IAPRTC-12 (CNN)}  \vline
\\
 & & provided & val & testing & provided & val & testing & provided & val & testing & provided & val & testing & provided & val & testing 
 \\
\hline \hline
\multirow{2}{*}{AP} & 
MLMG-CO & 0.9994 & 0.745 & 0.716 & 0.9994 & 0.4373 & 0.4422 & 0.9994 & 0.581 & 0.5887 & 1.0 & 0.5487 & 0.5527 & 0.9971 & 0.6455 & 0.6467
\\
& MLMG-SL & 1.0 & 0.7694 & 0.7234 & 0.9967 & 0.4435 & 0.4441 & 0.9997 & 0.5965 & 0.5946 & 1.0 & 0.5518 & 0.5535 & 0.9997 & 0.6615 & 0.6604
\\
\hline
\multirow{2}{*}{mAP} & 
MLMG-CO & 0.9999 & 0.5167 & 0.3344 & 1.0 & 0.2573 & 0.2458 & 0.9999 & 0.4544 & 0.4578 & 1.0 & 0.4443 & 0.4412 & 0.999 & 0.5273  & 0.509
\\
& MLMG-SL & 1.0 & 0.5404 & 0.3344 & 0.9996 & 0.2633 & 0.266 & 1.0 & 0.4949 & 0.4866 & 1.0 & 0.4472 & 0.4413 & 0.9999 &  0.5417 & 0.533
\\
\hline
\end{tabular}
}
\end{center}
\end{small}
\vspace{-0.04in}
\end{table*}

MLMG-CO + filling + constraint shows the best results among four variants at most cases. It not only gives the highest AP and mAP values, but also the semantically consistent results. 
This demonstrates that both filling and constraint contribute to the performance. 
Note that the improvements of AP values of the other three variants over MLMG-CO are larger than the improvements of mAP values. 
The main reason is that both filling and constraint directly influence the labels in each column, and AP measures the label ranking performance in each column. In contrast, the row ranking, which is measured by mAP, is indirectly influenced by filling and constraint through the label propagation on the mixed dependency graph.

\vspace{4pt}\noindent
\textbf{Comparison among four variants of MLMG-SL}. 
In Fig. \ref{fig: AP, mAP and AHL results of all cases}, the results of four variants of MLMG-CO are denoted using the lines with the 
$\triangleright$ mark, but with different colors. 
Similar with the above comparison about MLMG-CO, MLMG-SL shows the worst performance among its four variants; MLMG-SL + filling and MLMG-SL + constraint show similar performance in most cases; MLMG-SL + filling + constraint shows the best performance.

\vspace{4pt}\noindent
\textbf{Comparison between MLMG-CO and MLMG-SL}. 
In Fig. \ref{fig: AP, mAP and AHL results of all cases}, the corresponding variants of MLMG-CO and MLMG-SL are denoted as lines with the same color, but with different marks ($\circ$ and $\triangleright$ respectively, see the same column of the legend). 
Similar with the comparison between MLMG-CO and MLMG-SL shown in Section \ref{sec: 5 subsec results without semantic hierarchy}, on most datasets, MLMG-CO performs better than MLMG-SL when the missing label proportion is small, while worse when the missing label proportion increases. 

\vspace{4pt}\noindent
\textbf{Comparison between MLMG-CO+constraint and CSSAG}. 
Based on the input continuous labels produced by MLMG-CO, CSSAG will change continuous labels to binary ones according to the semantic hierarchy and the predefined number of positive labels.
Consequently, the AP results of the discrete outputs of CSSAG are similar to the AP values of MLMG-CO. 
But the mAP values of CSSAG are much lower than that of MLMG-CO. We think the reason is that CSSAG focuses on adjusting the column-wise label rankings, while mAP measures the row-wise label ranking performance. 
Moreover, although CSSAG ensures that there are no inconsistent labels in its binary label matrix, it cannot provide a consistent continuous label ranking. In contrast, ML-MG can satisfy these two conditions simultaneously. This comparison demonstrates that using the semantic hierarchy as constraint during optimization (as did in MLMG-CO + constraint) is more effective than using it as the constraint in the post-processing step (as did in CSSAG).

\subsection{Evaluation of Semi-supervised Multi-label Learning}

In above experiments, missing labels are randomly generated across different training instances and different classes. A special case is that some training instances are fully annotated, while other training instances are totally unlabelled, referred to as semi-supervised multi-label learning (SSML) \cite{semi-multi-label-sdm-2008}. Our proposed model can naturally handle SSML. In contrast, not all compared multi-label models that handle missing labels can exploit totally unlabelled images, such as FastTag \cite{fasttag-icml-2013}. 
Here we provide a further evaluation of our proposed methods in the SSML setting. 
Specifically, we randomly choose a subset of training instances, of which the size is equivalent to the size of the testing instance set, then hide their labels to the model ({\it i.e.}, setting the label value to $\frac{1}{2}$). This subset is referred to as {\it validation} set, while the subset of other fully labelled training images are called as {\it provided} set. The equivalent size between the validation set and the testing set ensures the fair comparison of the prediction performance on this two sets. 
For clarity, we only present the experiment at the case of not-filling initial label matrix and with SH constraint (see Table \ref{table: different algorithm names}). Besides, since MLMG-SL is inapplicable to Corel 5k with missing labels, here we ignore Corel 5k. 
The results are shown in Table \ref{table: ssl of MLMG-CO and MLMG-SL}. 
On both ESP Game and IAPRTC-12, the results evaluated by AP and mAP on the validation set are similar with or slightly higher than that on the testing set. It demonstrates that the joint probability distributions of image features and labels are close on training set and testing set of these two datasets. This point could facilitate to determine the model and algorithm parameters of our methods using cross-validation. 
However, on MediaMill, there are significant gaps between the evaluation results on the validation set and the testing set, especially the results evaluated by mAP. 
This reveals that the joint probability distributions of instance features and labels are different between the training and the testing set in MediaMill.

\begin{figure*}[!tbh]
\centering
\includegraphics[width=\textwidth, height=1.25in]{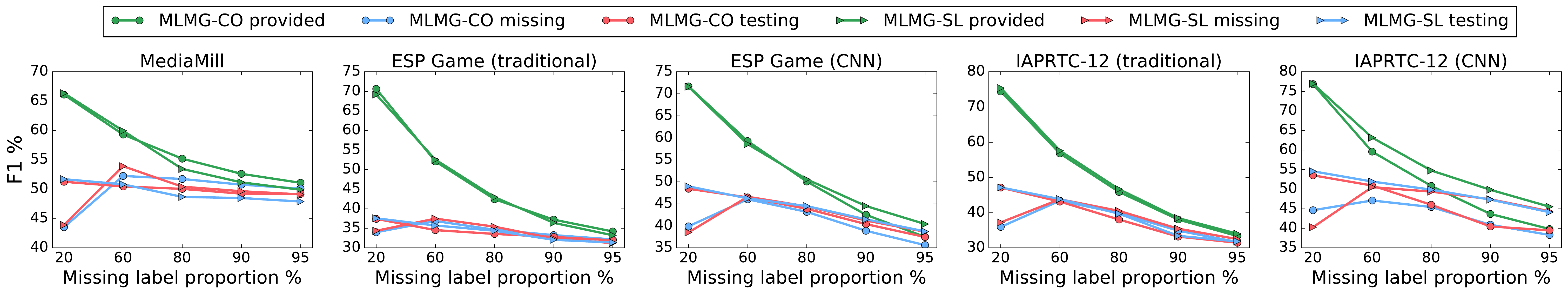}
\caption{F1 scores of the predicted labels of the provided, missing and testing entries in the label matrix. Please see Section \ref{sec: 5 subsec evaluation of missing imputation} for details. Figure better viewed on screen.}
\label{fig: F1 on given, miss, test labels}
\end{figure*}

\begin{figure*}[!tbh]
\centering
\includegraphics[width=\textwidth, height=2in]{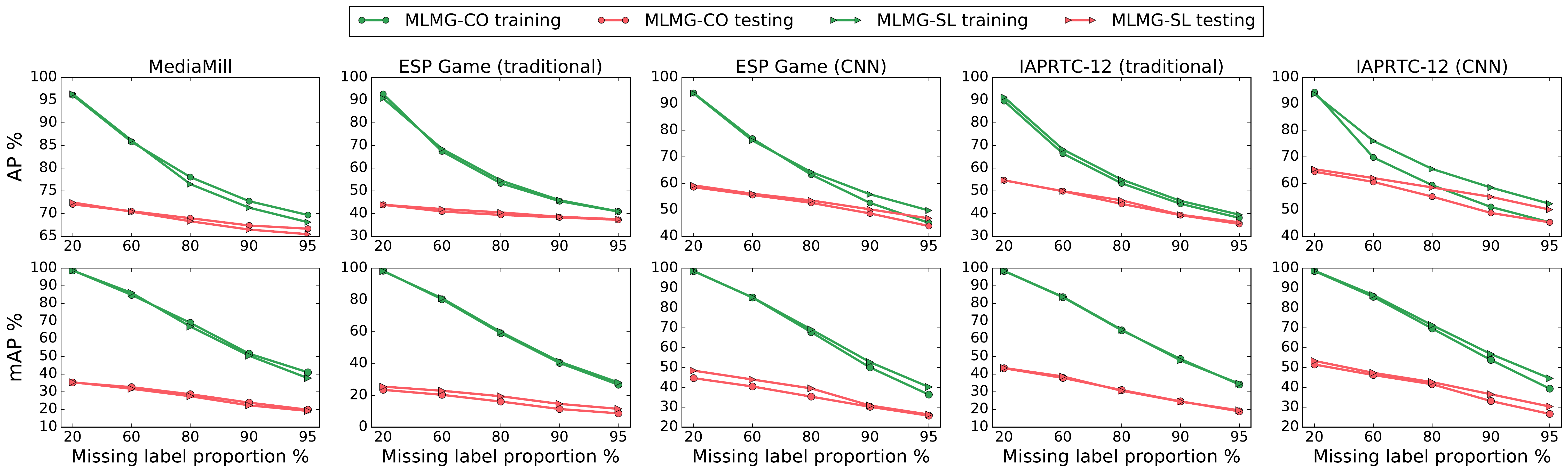}
\caption{Average precision (\textbf{top}) and mAP (\textbf{bottom}) results of our proposed methods on evaluation of missing label imputations. Please see Section \ref{sec: 5 subsec evaluation of missing imputation} for details. Figure better viewed on screen.}
\label{fig: AP and MAP results on all classes could be missing}
\end{figure*}

\subsection{Evaluation of Missing Label Imputations}
\label{sec: 5 subsec evaluation of missing imputation}

As transductive models, our proposed methods can not only predict the labels of testing images, but also impute the missed labels of training images. 
Here we evaluate the imputation performance of missing labels using our methods, and compare with the prediction performance of testing images. 
For clarity, we only present the experiment at the case of not-filling initial label matrix and with SH constraint (see Table \ref{table: different algorithm names}). 
In above experiments, missing labels are generated on only leaf and singleton classes. However, as our method enforces that the label score of the parent class cannot be lower than that of its child classes, the leaf and singleton classes are at the disadvantage in the competition with the root and intermediate classes. Thus, the imputation performance of missing labels corresponding to leaf and singleton classes is very poor, as they have to compete with the provided labels, of which a large proportion correspond to root and intermediate classes. 
Instead, here we change the setting of generating missing labels to that all classes could be missing. Then, at the same missing label proportion, there are actually more missing labels in the training label matrix, compared with the case that missing labels are only generated on leaf and singleton classes. 
Note that Corel 5k is not evaluated here. As demonstrated in Section \ref{sec: 5 subsec results without semantic hierarchy}, the SVD decomposition in MLMG-SL cannot give the valid solution on Corel 5k, due to the extremely sparse positive labels in the provided label matrix of Corel 5k, especially when missing labels on all classes exist. 

We present two evaluations. 
The first evaluation is using $F_1$ score, on provided, missing and testing labels. Specifically, we firstly discretize the predicted continuous label matrix by setting the labels of the top-10 largest scores in the label vector corresponding to each image ({\it i.e.}, the column vector of the label matrix) to $1$, while all other entries in the same label vector as $0$. The sub-label-vectors of both provided labels and missing labels are extracted from each training label vector in the binary label matrix, which could be evaluated using $F_1$ score separately. 
This evaluation could clearly reveal that imputation performance of missing labels, compared with the prediction performance on provided labels and testing labels.   
The results are shown in Fig. \ref{fig: F1 on given, miss, test labels}. 
There are two observations from the results on all datasets. 
One is that the prediction performance of provided labels (see the green lines in Fig. \ref{fig: F1 on given, miss, test labels}) is always better than that of missing and testing labels, but the performance advantage is inversely proportional to the missing label proportion. 
The reason is that the label consistency term in our model (see Eq. (\ref{eq: loss function})) encourages the predicted label scores to be consistent with the ground-truth labels at the provided entries of the label matrix. In contrast, there are no such a consistency term for missing and testing labels. 
When the missing label proportion is small, this consistency term could provide the reference for more labels.  
This explains the inverse proportion between the performance advantage of the prediction on provided labels and the missing label proportion. 
The other observation is that the imputation performance of missing labels (see the blue lines in Fig. \ref{fig: F1 on given, miss, test labels}) is worse than that of testing labels (see the red lines in Fig. \ref{fig: F1 on given, miss, test labels}) at the missing label proportion $20\%$, but their performance becomes similar when the missing label proportion is large.
The missing labels have to compete with the provided labels in the same label column. When a large proportion of labels are provided, the unfairness between missing labels and provided labels may preclude the recovery of the ground-truth positive labels in missing labels. The degree of this unfairness is inversely proportional to the missing label proportion. 
In contrast, there is no such a unfairness among the entries in the same label column for testing images, as all entries in the same column are missing. 
This difference is the main reason that the performance gap between the imputations of missing and testing labels when the missing label proportion is small. 

\begin{table*}[phtb] 
\caption{ Runtime in seconds of all compared methods. The smallest runtime of each dataset is highlighted in bold.}
\label{table: running time}
\begin{center}
\scalebox{1.02}{
\begin{tabular}{|l|cccccccc|}
\hline
\multirow{2}{*}{Datasets} & MC-Pos & FastTag & LEML & MLML & MLML & MLMG-CO & MLMG-CO + & MLMG-SL +
\\
 & \cite{MC-Pos-nips-2011} & \cite{fasttag-icml-2013} & \cite{LEML-ICML-2014} & -exact \cite{my-icpr-2014} & -appro \cite{my-icpr-2014} &  & constraint & constraint
 \\
 \hline \hline
 Corel 5k & 72.94 & 55.94 & 203.4 & 50.6 & 2.56 &\textbf{1.86} & 4.76 & 21.3
 \\
 MediaMill & 165.4 & 56.5 & 151.5 & 238.4 & \textbf{4.63} & 8.75 & 23.39 & 44.25
 \\
 ESP Game (traditional) & 337 & 124 & 638.2 & \multirow{2}{*}{3004} & \multirow{2}{*}{239.2} & \multirow{2}{*}{\textbf{11}} & \multirow{2}{*}{28.2} & \multirow{2}{*}{337.3} 
 \\
 ESP Game (CNN) & 1772  & 199.3 & 2164 &  &  &  &  & 
 \\
 IAPRTC-12 (traditional) & 326.6 & 202.2 & 742.2 & \multirow{2}{*}{3378} & \multirow{2}{*}{238.4} & \multirow{2}{*}{\textbf{11.6}} & \multirow{2}{*}{30.5} & \multirow{2}{*}{158}
 \\
 IAPRTC-12 (CNN) & 1709 & 271.7 & 2063 &  &  &  &  & 
 \\
 \hline
\end{tabular}
}
\end{center}
\end{table*}

The second evaluation is using the metrics AP and mAP, on both training and testing images, as shown in Fig. \ref{fig: AP and MAP results on all classes could be missing}.  It provides the observation of the performance influence of the additional missing labels at the root and intermediate classes. 
Compared with the reported results in Fig. \ref{fig: AP, mAP and AHL results of all cases} (see MLMG-CO+constraint and MLMG-SL+constraint), the corresponding results of 
ML MG-CO and MLMG-SL in Fig. \ref{fig: AP and MAP results on all classes could be missing} are slightly lower. The reason is that the additional missed labels at root and intermediate classes could be easily recovered using our methods, if any one of their descendant classes are provided.

\subsection{Complexity and Runtime}
\label{sec: subsec complexity and runtime}

\noindent
{\bf Complexity}.
Here we analyze the complexities of our methods. 
MLMG-CO is implemented by the PGD algorithm (see Section \ref{sec4: subsec ADMM for MLMG-CO}), which can be further accelerated with the following observations.
First, both $\mathbf{L}_\X$ and $\mathbf{L}_\C$ are sparse, and there are only $n k_\X $ and $m k_\C$ non-zero entries, respectively. $k_\X \ll n$ denotes the number of neighbours at the instance-level, while $k_\C \ll m$ is the number of neighbours at the class-level (their specific values on different datasets are shown in Table \ref{table: dataset}).
Second, there are some shared terms between different steps, such as $\Z \mathbf{B}_k$ and $\C_k \Z$.
Third, it is known that $\tr(abc) = \tr(bca) = \tr(cab)$. Thus we have $\tr(\A_k^\top \Z) = \tr(\Z \A_k^\top)$ or $\tr(\Z^\top \C_k \Z) = \tr(\Z \Z^\top \C_k)$. Considering $m \ll n$ always holds in the datasets in our experiments, the computational cost can be significantly reduced from $O(mn^2)$ to $O(m^2n)$, or from $O(mn^2 + m^2n)$ to $O(m^2n + m^3)$.
Utilizing the above three observations, the actual computational complexity of MLMG-CO is $O_{CO}$ $ = O( T_1 ( (2k_X+ k_C)mn + 3k_C m^2 + 8 m^2 n ))$, with $T_1$ being the number of iterations of the PGD algorithm, and $T_1 \leq 50$ in our experiments.
Consequently, the computational complexity of MLMG-CO + constraint (implemented by ADMM algorithm) is 
 $O_{CO+constraint} = O(T_2 ( O_{CO} + 2 n_e m n) )$, with $T_2$ being the number of iterations of the ADMM algorithm. As shown in Section \ref{sec: 5 subsec convergence curve}, $T_2$ is small in most cases ($T_2\leq 20$). Besides, when calling the PGD algorithm (for updating $\Z$) in ADMM, $T_1$ is set to a small value (less than $5$) and in which case the ADMM algorithm always achieves very good performance in a few iterations. 
Compared to MLMG-CO + constraint, the main additional computational cost of MLMG-SL is the SVD step when updating $\HH_0$ (see Eq. (\ref{eq: update of H0 in admm for MLMG-SL})), which takes $O(m n^2)$ in each iteration of ADMM. Thus the computational complexity of MLMG-SL is $O_{SL} = O(T_2$ $( O_{CO} + 2 n_e m n + m n^2) )$. The complexities of MLMG-SL + filling, MLMG-SL + constraint, MLMG-SL + filling + constraint is similar with that of MLMG-SL.
In terms of the space complexity, since $\mathbf{L}_\X$ and $\mathbf{L}_\C$ are sparse, the largest size of all matrices is only $mn$. So, the overall space complexities are $O(nk_X + mk_C + 11mn + 5m^2 + n_e(2n + n_e))$ of MLMG-CO + constraint and $O(nk_X + 15mn + m^2 + n_e(2n + n_e))$ of MLMG-SL, respectively.

\vspace{1pt}\noindent
{\bf Runtime}.
To demonstrate the computational efficiency of our methods, we report the average runtime of all methods on four datasets in Table \ref{table: running time}. 
All experiments are run on a Linux machine with Intel(R) Xeon(R) CPU E5-2680 v4 2.40GHz. 
In the case of $50\%$ missing labels, each method is run 10 times and the average runtime is recorded.
For all varied algorithms of MLMG-CO and MLMG-SL (see Table \ref{table: different algorithm names}), except for MLMG-CO, and LEML, the number of maximum iterations is set as 20, while that of MLMG-CO is set as 50. 
The ranks of the mapping matrix in LEML are set as 50, 50, 50 and 20 for Corel 5k, ESP Game, IAPRTC-12 and MediaMill, respectively. 
Note that the last five methods in Table \ref{table: running time} are transductive methods, of which the runtime is independent with the feature dimension, thus their runtime on traditional and CNN features is same. 
Besides, since the filling variants of MLMG-CO and MLMG-SL will not influence the computational complexity, thus their corresponding runtime is not presented. 
Clearly, our proposed algorithms (especially MLMG-CO and MLMG-CO + constraint) are significantly more computationally attractive than most compared methods.

\begin{figure*}[hptb]
\centering
\includegraphics[width=\textwidth,height=1.85in]{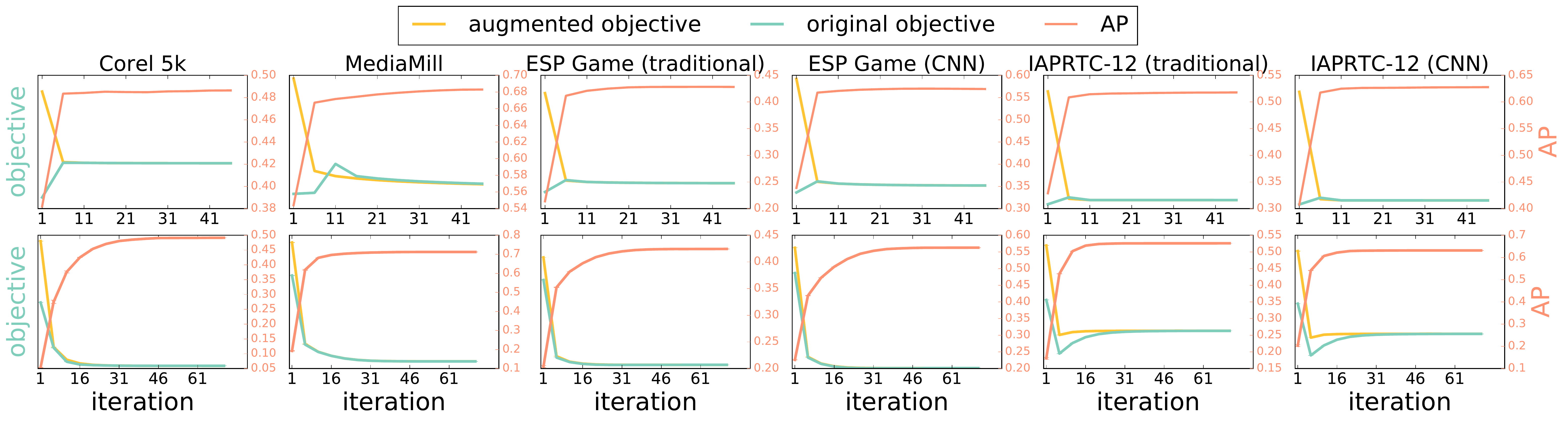}
\caption{Convergence curves of MLMG-CO with (\textbf{Top}): initializing missing labels as 0; (\textbf{Bottom}): initializing missing labels as random (10 times) values in $[0,1]$.
`objective' indicates the value of the (augmented/original) functions. To save space, we hide the objective values at the left vertical axis. 
AP indicates the evaluation value of average precision.
In the bottom row, as the std values are very small compared to the mean values, it is better to enlarge the figure to check them. }
\label{fig: Convergence curve random initialization of MLMG-CO}
\end{figure*}

\begin{figure*}[hptb]
\centering
\includegraphics[width=\textwidth,height=1.85in]{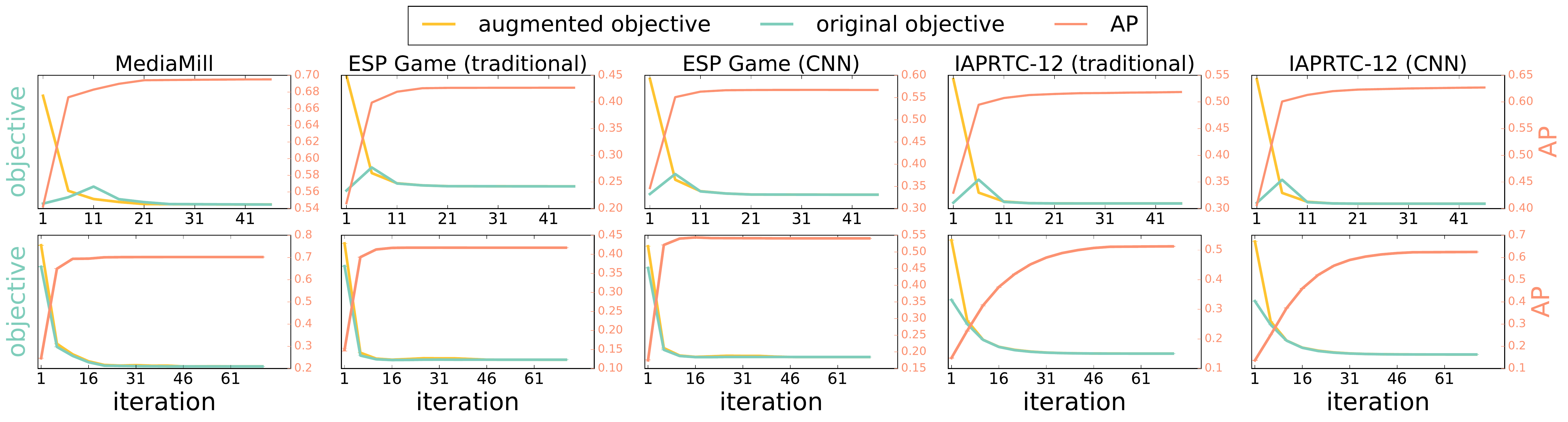}
\caption{Convergence curve of MLMG-SL with (\textbf{Top}): initializing missing labels as 0; (\textbf{Bottom}): initializing missing labels as random (10 times) values in $[0,1]$.}
\label{fig: Convergence curve random initialization of MLMG-SL}
\end{figure*}

\subsection{Convergence and Sensitivity to Label Initialization} 
\label{sec: 5 subsec convergence curve}
Here we evaluate the convergence of MLMG-CO + constraint (denoted as MLMG-CO for clarity in this subsection) and MLMG-SL + constraint (denoted as MLMG-SL for clarity in this subsection) using different label initializations. We only present the curve in the case of $50\%$ missing labels, as shown in Fig. \ref{fig: Convergence curve random initialization of MLMG-CO} for MLMG-CO. 
In the top row of both figures, we initialize the label matrix $\mathbf{Z}$ by setting all missing labels as $0$. In this case, MLMG-CO converges to its best AP value in less than 10 iterations on all datasets. 
In the bottom row, missing labels are initialized as random values in $[0,1]$. We repeat MLMG-CO with 10 random initializations and report the mean and standard deviation (std) values. The extremely small std values of both objective functions and AP values suggest MLMG-CO is insensitive to initialization.
Similarly, the convergence analysis of MLMG-SL is also shown in Fig. \ref{fig: Convergence curve random initialization of MLMG-SL}. 
It also verifies the fast convergence and the robustness to initialization of MLMG-SL. 
As demonstrated before, MLMG-SL cannot give results on Corel 5k when missing training labels exist, the convergence analysis on Corel 5k are not presented here.

\begin{figure}[t]
\centering
\includegraphics[width=0.49\textwidth, height=7in]{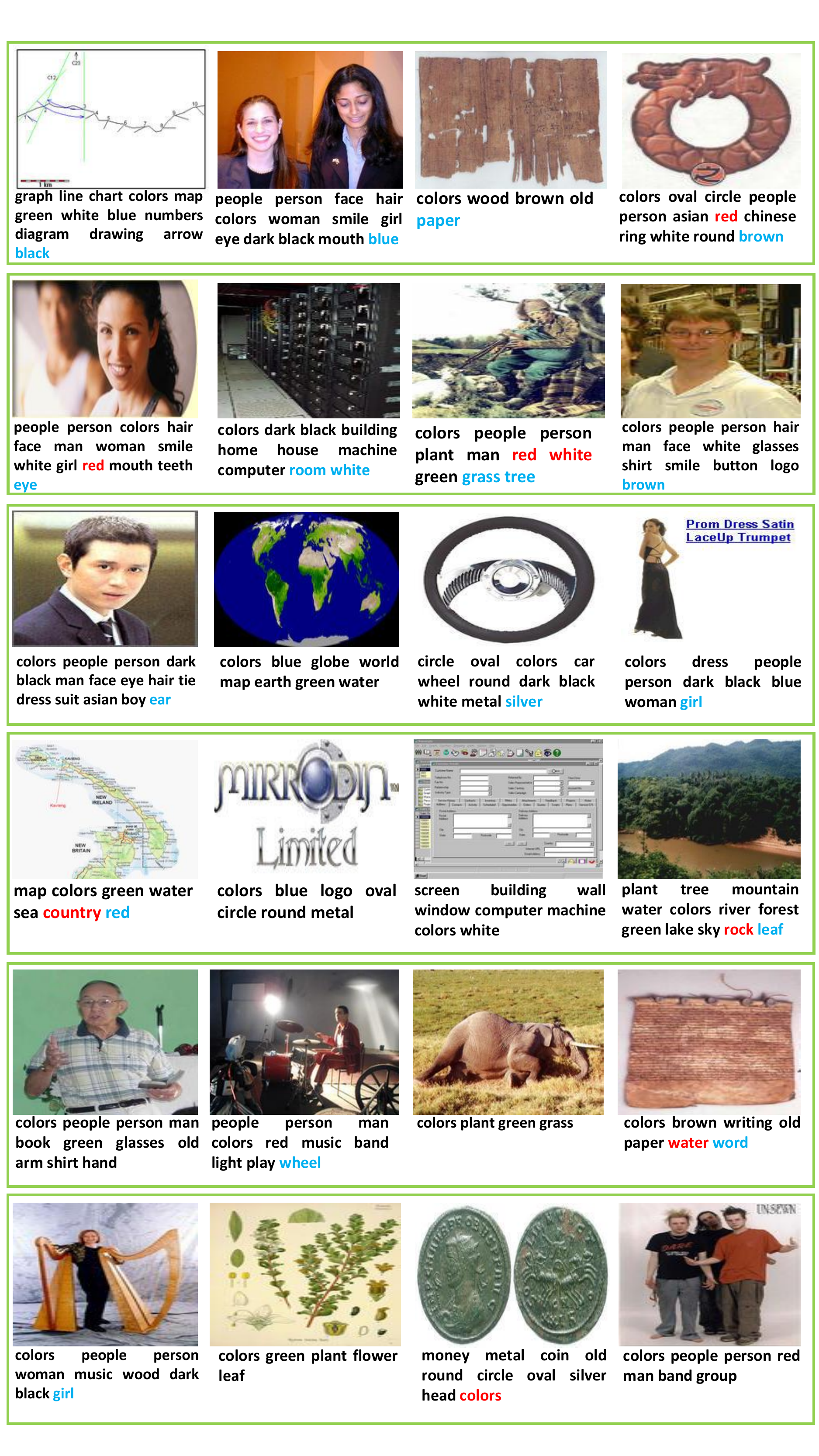}
\vspace{-0.15in}
\caption{Some image annotation results of MLMG-SL + filling + constraint in the case of $20\%$ missing training labels on ESP Game. The incorrectly predicted tags are highlighted in \red{red} color, while the missed ground-truth tags are highlighted in \blue{blue} color.}
\label{fig: quality results of image annotation on espgame}
\end{figure}

\begin{figure}[t]
\centering
\includegraphics[width=0.49\textwidth, height=7in]{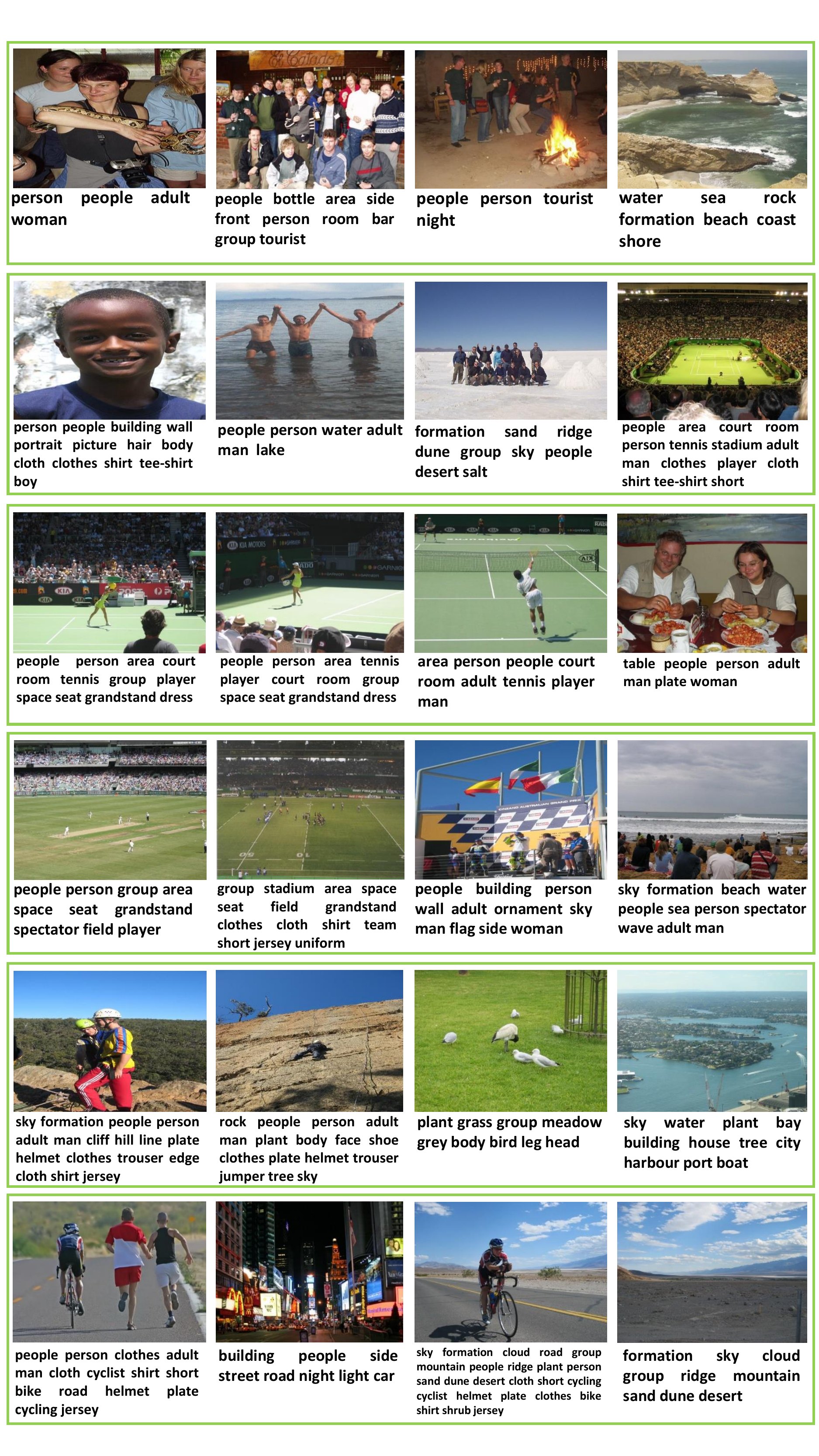}
\vspace{-0.15in}
\caption{Some image annotation results of MLMG-SL + filling + constraint in the case of $20\%$ missing training labels on IAPRTC-12. All above quality results are $100\%$ correct.}
\label{fig: quality results of image annotation on iaprtc12}
\end{figure}

\subsection{Qualitative Results}
\label{sec: subsec qualitative results}

We present some qualitative results of our method MLMG-SL + filling + constraint, as shown in Fig. \ref{fig: quality results of image annotation on espgame} and \ref{fig: quality results of image annotation on iaprtc12} for image annotation, as well as the tag-based image retrieval results in Figs. \ref{fig: quality results of image retrieval on espgame} and \ref{fig: quality results of image retrieval on iaprtc12}.
It is worthwhile to note that parent classes are always ranked higher than their children classes in image annotation results, which is the result of employing semantic hierarchies.

Note that it seems that there are more errors on color tags ({\it e.g.}, `colors', `red', `blue', `green', `white', `brown') than on other object tags in the presented results in Fig. \ref{fig: quality results of image annotation on espgame}. We find that the low quality of ground-truth on these color tags in ESP Game is an important reason. For example, we predict `red' in R1-C4 ({\it i.e.}, the image in row 1 and column 4), `white' in R2-C3, `colors' in R5-C3. These tags are actually correct, but they are missed in the ground-truth. Moreover, the low quality of the ground-truth for color tags will degrade the performance of our methods on these tags. We believe the reason is that  color tags are attribute tags, and they are less important for human annotators than other object tags, during collecting ground-truth tags. 
Unfortunately, these are still many missed color tags in the complete ground-truth label matrix of ESP Game (see Table \ref{table: dataset}), as the semantic hierarchical constraint cannot be help to augment the missed tags on leaf nodes ({\it e.g.}, `red', `blue', `green', `white' and `brown').

\begin{figure*}[phtb]
\centering
\includegraphics[width=\textwidth, height=3.5in]{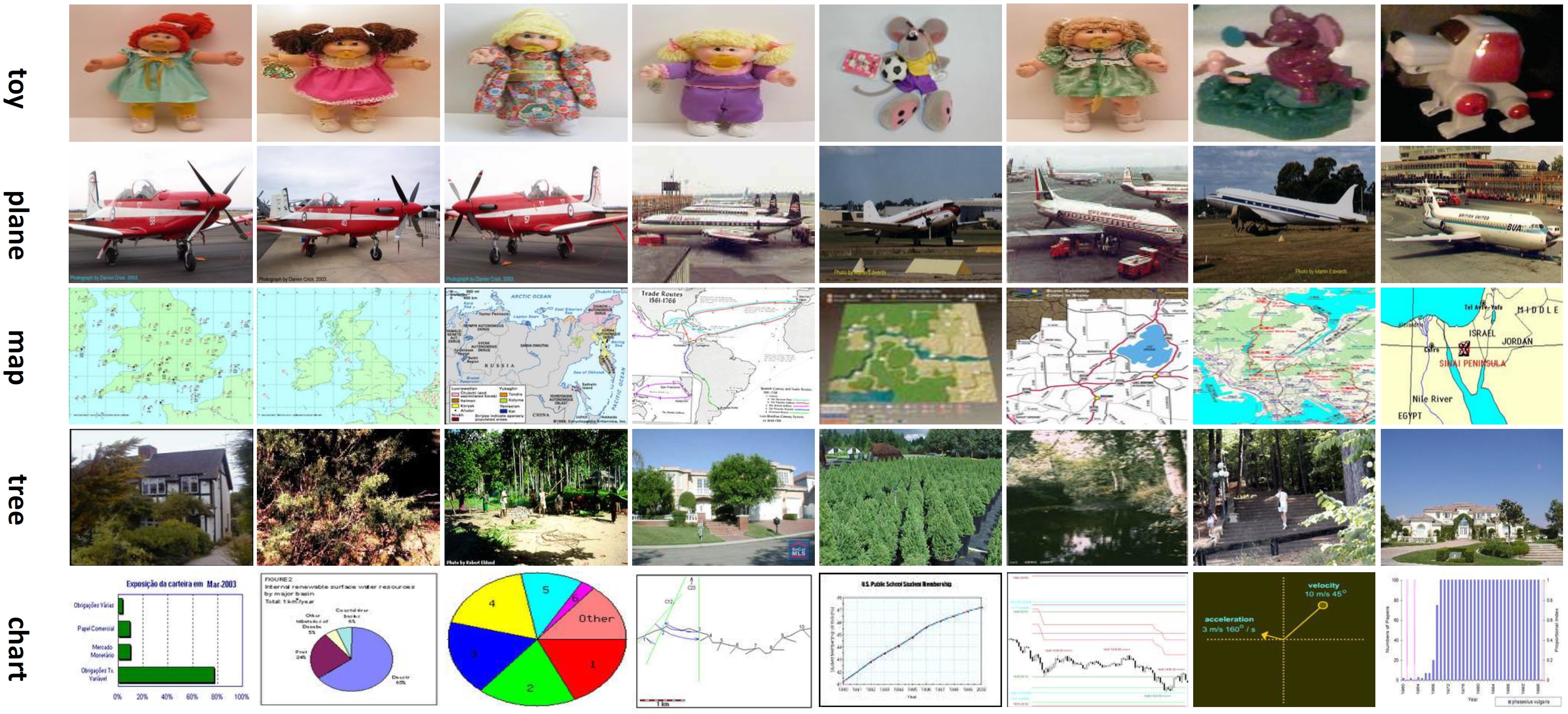}
\vspace{-0.15in}
\caption{Some tag-based image retrieval results of MLMG-SL + filling + constraint on ESP Game.}
\label{fig: quality results of image retrieval on espgame}
\end{figure*}

\begin{figure*}[phtb]
\centering
\includegraphics[width=\textwidth, height=3.5in]{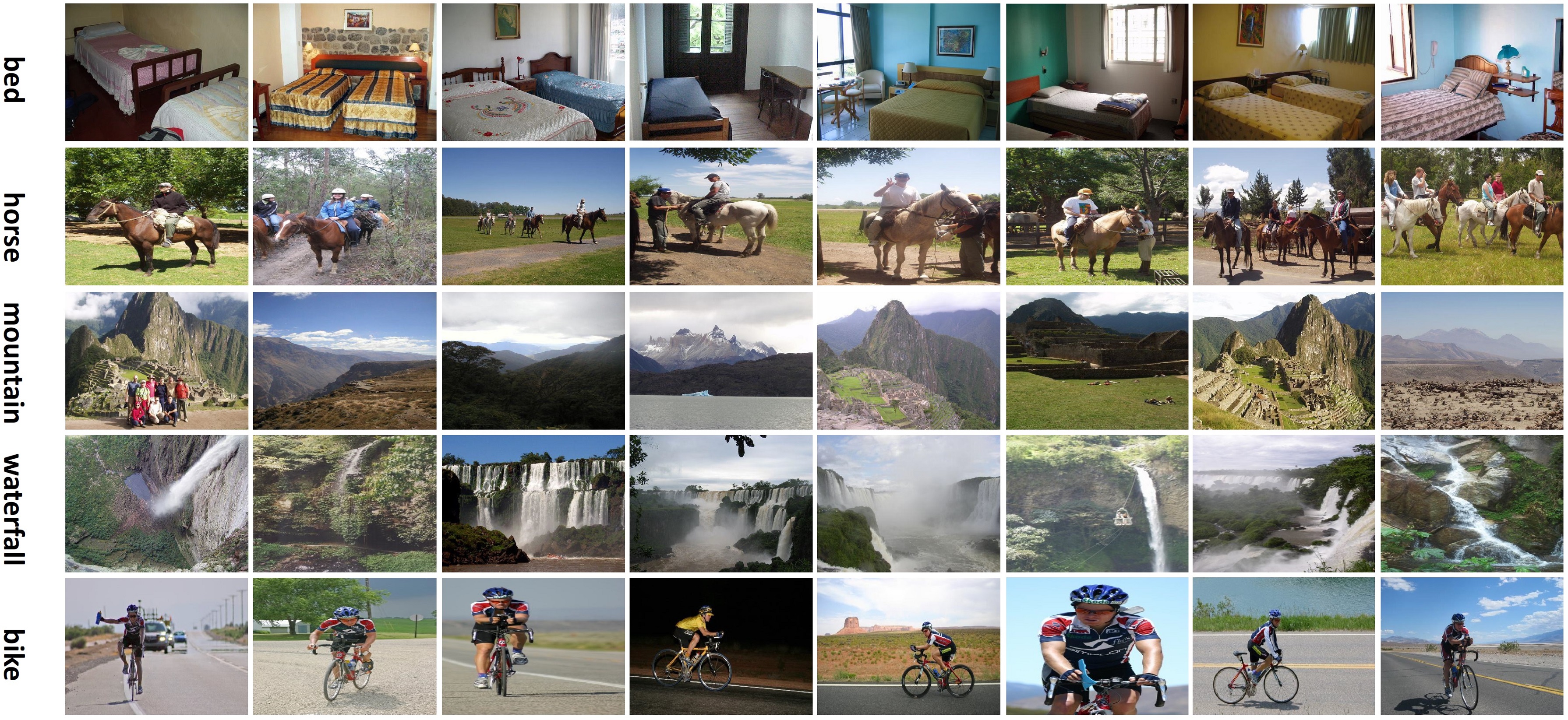}
\vspace{-0.15in}
\caption{Some tag-based image retrieval results of MLMG-SL + filling + constraint on IAPRTC-12.}
\label{fig: quality results of image retrieval on iaprtc12}
\end{figure*}

\section{Discussions} \label{sec: discussion} 

In Section \ref{sec: 5 experiments}, we have presented detailed experimental comparisons between MLMG-CO and MLMG-SL, as well as their variants. 
Here we give a brief summary about their performance. 
Among four variants of MLMG-CO, MLMG-CO + filling + constraint shows the best performance in most cases, as it not only gives the higher AP and MAP values, but also the semantically consistent label ranking for each instance. 
Similarly, MLMG-SL + filling + constraint is also the best choice among the variants of MLMG-SL. 
In terms of the comparison between MLMG-CO and MLMG-SL:
when there are few missing labels in the training label matrix, their performance is similar, and MLMG-CO is slightly better in more cases; 
as the missing label proportion increases, MLMG-SL shows better performance than MLMG-CO in most cases. The main reason is the class-level smoothness assumption used in MLMG-CO is significantly influenced by missing labels, and when the missing label proportion is large, the derived smoothness is possible to be inaccurate. In contrast, the sparse and low rank decomposition used in MLMG-SL is independent with missing labels.  
Besides, as shown in Section \ref{sec: subsec complexity and runtime}, MLMG-SL is of the higher computational cost than MLMG-CO. 
Thus, our suggestion is: when few missing labels occur in the training label matrix, MLMG-CO + filling + constraint is the better choice; when many missing labels exist, MLMG-SL + filling + constraint is preferred.

Here we also present a discussion about the advantages and limitations of our propose methods. Both MLMG-CO and MLMG-SL are transductive models, as they propagate the label information between training and testing instances through the instance-level label dependency, and directly predict the label matrix, without a parametric classifier. In experiments we find that the instance-level label dependency contributes a lot to the prediction performance. However, when handling new testing images, the whole set of training images have to be exploited. Thus it is difficult to apply to very large scale or streaming datasets, which is the main limitation of transductive models. 
In contrast, the inductive model with a parametric classifier is easy to predict the labels of new testing instances, thus it is applicable for very large scale or streaming datasets. However, the label dependency among instances is ignored, and the prediction performance may degrade. Besides, as there are much more parameters of the parametric classifier, it requires more training instances to achieve satisfied performance. 
In short, the proposed transductive models evaluated in this manuscript show very superior performance on modest sized datasets, but it is not suitable for very large scale or streaming datasets; the inductive model is suitable for large scale or streaming datasets, but with the possible performance degradation due to the removal of the instance-level label dependency. 
The best approach is combining their advantages into one unified model, while alleviating their limitations. One attempt has been proposed in one of our previous work \cite{my-pr-2015}. The main idea is to replace each entry of the label matrix $\Z$ by a parametric mapping function, {\it i.e.}, $\Z_{ij} = f(\w_i^\top \x_j)$, where $\w_i$ indicates the parameter vector corresponding to the classifier of class $i$. Its computational cost is much higher than that of our proposed methods in this manuscript. 
It deserves more efforts to find a better approach to combine the advantages of transductive and inductive models together, to achieve the better trade-off between the efficacy and efficiency. 

\section{Conclusions}\label{sec: conclusion}

This work has proposed a novel model to handle the problem of multi-label learning with missing labels (MLML). A unified network of label dependencies is built based on a mixed dependency graph to propagate the label information from provided labels to missing labels. 
We construct two types of mixed dependency graphs, including the mixed graph with co-occurrence (MG-CO) and that with sparse and low rank decomposition (MG-SL). 
Both MG-CO and MG-SL utilize the ins- tance-level similarity as undirected edges to connect the label nodes across different instances, as well as the semantic hierarchy as directed edges to connect different classes.
Additionally, another pairwise label dependency called class co-occurrence is also used in MLMG-CO to connect the label nodes corresponding to the same instance. 
In contrast, the high-order label dependency through sparse and low rank decomposition of the whole label matrix is adopted in MLMG-SL to implicitly connect all classes together. 
Based on this two mixed dependency graphs respectively, the MLML problem is formulated as two convex optimization problems, both of which can be efficiently solved by the ADMM algorithm. 
Due to the joint utilization of the instance-level and class-level label dependency, our methods can simultaneously output the image annotation and the tag-based image retrieval results. 
Experimental evaluations on four benchmark datasets have demonstrated the superior performance of our methods against state-of-the-art methods. Moreover, we contribute manually generated semantic hierarchies, as well as the corresponding semantically augmented ground-truth label matrices, for four popular benchmark datasets, which will be beneficial to the research community at large.
These semantic hierarchies, augmented gro-und-truths and the codes of the proposed methods will be available at {\it https://sites.google.com/site/baoyuanwu2015/}.

\appendix

\section{Convexity proof}
\label{sec1: convexity proof}

For clarity and continuity, we rewrite the continuous optimization problem of MLMG-CO here,
\begin{flalign}
  \min_\Z & \hspace{0.05in} f_1(\Z) = -\tr( \overline{\Y}^\top \Z ) + \beta \tr(\Z \mathbf{L}_\X \Z^\top) + \gamma \tr(\Z^\top \mathbf{L}_\C \Z), \nonumber
  \\
 \text{s.t.}   & \hspace{0.05in} \Z\in [0, 1]^{m \times n},
  \quad \boldsymbol{\Phi}^\top \Z  \geq 0.
  \label{eq: obj of matrix based continuous}
\end{flalign}

Firstly we introduce the following vector variables:
\begin{flalign}
& \z = \text{vec}(\Z) = [\Z_{11}, \ldots, \Z_{m1}, \ldots, \Z_{mn}]^\top \in \{ -1, +1 \}^{mn \times 1}, \nonumber \\
& \overline{\y} = \text{vec}(\overline{\Y})= [\overline{\Y}_{11}, \ldots, \overline{\Y}_{m1}, \ldots, \overline{\Y}_{mn}]^\top \in \{ -1, 0, +1 \}^{mn}, \nonumber\\
& \W = \beta \cdot \W_\X^\top \otimes \I_m + \gamma \cdot \I_n \otimes \W_\C\in \mathbb{R}^{mn \times mn}, \nonumber\\
& \mathbf{L} = \beta \cdot \mathbf{L}_\X^\top \otimes \I_m + \gamma \cdot \I_n \otimes \mathbf{L}_\C   \in \mathbb{R}^{mn \times mn}, \nonumber\\
& \overline{\boldsymbol{\Phi}} =  \boldsymbol{\Phi} \otimes \I_n, \nonumber
\end{flalign}
where $\otimes$ indicates the Kronecker product \cite{kronecker-product-1958}.
Then Problem (\ref{eq: obj of matrix based continuous}) can be transformed to its equivalent vector based formulation, as follows:
 \begin{flalign}
  \arg\min_\z & \quad f_2(\z) = -\overline{\y}^\top \z  + \z^\top \mathbf{L} \z,
    \label{eq: obj of vector based continuous}
  \\
 \text{s.t.}   & \quad \z\in [0, 1]^{mn \times 1},
  \quad \overline{\boldsymbol{\Phi}}^\top \z  \geq 0. \nonumber
\end{flalign}

\begin{lemma}
$\mathbf{L}$ is positive semi-definite (PSD).
\label{lemma: property of L}
\end{lemma}

\begin{proof}
Given two square matrix $\mathbf{A} \in \mathbb{R}^{n_1 \times n_1}$ and $\mathbf{B} \in \mathbb{R}^{n_2 \times n_2}$, their eigenvalues are denoted as $\lambda_1,\ldots,\lambda_{n_1}$ and $\mu_1,\ldots,\mu_{n_2}$.
According to the property of Kronecker product, the eigenvalues of $\mathbf{A} \otimes \textbf{B}$ are $\lambda_i \mu_j, i=1,\ldots,n_1; j = 1,\ldots,n_2$.
 $\mathbf{L}_\X^\top$ is PSD and $\mathbf{I}_m$ is positive definite (PD). Obviously all eigenvalues of $\mathbf{L}_\X^\top \otimes \mathbf{I}_m$ are non-zero values, so $\mathbf{L}_\X^\top \otimes \mathbf{I}_m$ is a PSD matrix. Similarly we can obtain that $\mathbf{I}_n \otimes \mathbf{L}_\C$ is also PSD. Finally, as $\mathbf{L}$ is the positive weighted linear combination of two PSD matrices, it is easy to conclude that $\mathbf{L}$ is a PSD matrix.
\end{proof}

\begin{proposition}
 Problem (\ref{eq: obj of matrix based continuous}) is convex.  
\label{proposition: proof of convexity}
\end{proposition}

\begin{proof}
The Hessian of the objective function in (\ref{eq: obj of vector based continuous}) with respect to $\z$ is $\mathbf{L}$, which has been proven to be a PSD matrix in Lemma \ref{lemma: property of L}. Thus, the objective function (\ref{eq: obj of vector based continuous}) is a convex function in $\z$.
The box and linear inequality constraints lead to a convex feasible solution space (that satisfies Slater's condition), so it is easy to conclude that Problem (\ref{eq: obj of vector based continuous}) is a convex optimization problem. Finally, as Problems (\ref{eq: obj of matrix based continuous}) and (\ref{eq: obj of vector based continuous}) are equivalent, thus Problem (\ref{eq: obj of matrix based continuous}) is a convex problem.
\end{proof}

\bibliographystyle{spmpsci}      
\bibliography{bywu_bib}   

\begin{thebibliography}{10}
\providecommand{\url}[1]{{#1}}
\providecommand{\urlprefix}{URL }
\expandafter\ifx\csname urlstyle\endcsname\relax
  \providecommand{\doi}[1]{DOI~\discretionary{}{}{}#1}\else
  \providecommand{\doi}{DOI~\discretionary{}{}{}\begingroup
  \urlstyle{rm}\Url}\fi

\bibitem{Agrawal-ml-million-label-www-2013}
Agrawal, R., Gupta, A., Prabhu, Y., Varma, M.: Multi-label learning with
  millions of labels: Recommending advertiser bid phrases for web pages.
\newblock In: WWW, pp. 13--24 (2013)

\bibitem{bi-wei-icml-2011}
Bi, W., Kwok, J.T.: Multi-label classification on tree-and dag-structured
  hierarchies.
\newblock In: ICML, pp. 17--24 (2011)

\bibitem{admm-boyd-2011}
Boyd, S., Parikh, N., Chu, E., Peleato, B., Eckstein, J.: Distributed
  optimization and statistical learning via the alternating direction method of
  multipliers.
\newblock Foundations and Trends in Machine Learning \textbf{3}(1), 1--122
  (2011)

\bibitem{boyd-convex-2004}
Boyd, S., Vandenberghe, L.: Convex optimization.
\newblock Cambridge university press (2004)

\bibitem{bucak-multi-incomplete-2011}
Bucak, S.S., Jin, R., Jain, A.K.: Multi-label learning with incomplete class
  assignments.
\newblock In: CVPR, pp. 2801--2808. IEEE (2011)

\bibitem{MC-Pos-nips-2011}
Cabral, R.S., De~la Torre, F., Costeira, J.P., Bernardino, A.: Matrix
  completion for multi-label image classification.
\newblock In: NIPS, pp. 190--198 (2011)

\bibitem{L1-label-denoising-bmvc-2016}
Chang, X., Xiang, T., Hospedales, T.M.: L1 graph based sparse model for label
  de-noising.
\newblock In: BMVC (2016)

\bibitem{vggf-bmvc-2014}
Chatfield, K., Simonyan, K., Vedaldi, A., Zisserman, A.: Return of the devil in
  the details: Delving deep into convolutional nets.
\newblock In: BMVC (2014)

\bibitem{ADMM-multiblock-not-convergent-2016}
Chen, C., He, B., Ye, Y., Yuan, X.: The direct extension of admm for
  multi-block convex minimization problems is not necessarily convergent.
\newblock Mathematical Programming \textbf{155}(1-2), 57--79 (2016)

\bibitem{semi-multi-label-sdm-2008}
Chen, G., Song, Y., Wang, F., Zhang, C.: Semi-supervised multi-label learning
  by solving a sylvester equation.
\newblock In: SIAM international conference on data mining, pp. 410--419 (2008)

\bibitem{fasttag-icml-2013}
Chen, M., Zheng, A., Weinberger, K.: Fast image tagging.
\newblock In: ICML, pp. 1274--1282 (2013)

\bibitem{multilabel-link-prediction-aaai2015}
Chen, Z., Chen, M., Weinberger, K.Q., Zhang, W.: Marginalized denoising for
  link prediction and multi-label learning.
\newblock In: AAAI (2015)

\bibitem{deng-eccv-2014}
Deng, J., Ding, N., Jia, Y., Frome, A., Murphy, K., Bengio, S., Li, Y., Neven,
  H., Adam, H.: Large-scale object classification using label relation graphs.
\newblock In: ECCV, pp. 48--64. Springer (2014)

\bibitem{imagenet-cvpr-2009}
Deng, J., Dong, W., Socher, R., Li, L.J., Li, K., Fei-Fei, L.: Imagenet: A
  large-scale hierarchical image database.
\newblock In: CVPR, pp. 248--255. IEEE (2009)

\bibitem{corel5k-eccv-2002}
Duygulu, P., Barnard, K., de~Freitas, J.F., Forsyth, D.A.: Object recognition
  as machine translation: Learning a lexicon for a fixed image vocabulary.
\newblock In: ECCV, pp. 97--112. Springer (2002)

\bibitem{nuclear-norm-low-rank-2002}
Fazel, M.: Matrix rank minimization with applications.
\newblock Ph.D. thesis, PhD thesis, Stanford University (2002)

\bibitem{wordnet-1998}
Fellbaum, C.: WordNet.
\newblock Wiley Online Library (1998)

\bibitem{ML-calibrated-ranking-2008}
F{\"u}rnkranz, J., H{\"u}llermeier, E., Menc{\'\i}a, E.L., Brinker, K.:
  Multilabel classification via calibrated label ranking.
\newblock Machine learning \textbf{73}(2), 133--153 (2008)

\bibitem{video-annotation-icm-2008}
Geng, B., Yang, L., Xu, C., Hua, X.S.: Collaborative learning for image and
  video annotation.
\newblock In: Proceedings of the 1st ACM international conference on Multimedia
  information retrieval, pp. 443--450. ACM (2008)

\bibitem{admm-proof-2013}
Ghadimi, E., Teixeira, A., Shames, I., Johansson, M.: Optimal parameter
  selection for the alternating direction method of multipliers (admm):
  quadratic problems.
\newblock IEEE Transactions on Automatic Control \textbf{60}(3), 644--658
  (2015)

\bibitem{ML-reivew-2014}
Gibaja, E., Ventura, S.: Multi-label learning: a review of the state of the art
  and ongoing research.
\newblock Wiley Interdisciplinary Reviews: Data Mining and Knowledge Discovery
  \textbf{4}(6), 411--444 (2014)

\bibitem{MC-nips-2010}
Goldberg, A.B., Zhu, X., Recht, B., Xu, J.M., Nowak, R.D.: Transduction with
  matrix completion: Three birds with one stone.
\newblock In: NIPS, pp. 757--765 (2010)

\bibitem{iaprtc-12-data-2006}
Grubinger, M., Clough, P., M{\"u}ller, H., Deselaers, T.: The iapr tc-12
  benchmark: A new evaluation resource for visual information systems.
\newblock In: International Workshop OntoImage, pp. 13--23 (2006)

\bibitem{multilabel-dataset-image-iccv-2009}
Guillaumin, M., Mensink, T., Verbeek, J., Schmid, C.: Tagprop: Discriminative
  metric learning in nearest neighbor models for image auto-annotation.
\newblock In: ICCV, pp. 309--316 (2009)

\bibitem{multilabel-compressed-sensing-nips-2012}
Kapoor, A., Viswanathan, R., Jain, P.: Multilabel classification using bayesian
  compressed sensing.
\newblock In: NIPS, pp. 2654--2662 (2012)

\bibitem{crbm-mlml-2015}
Li, X., Zhao, F., Guo, Y.: Conditional restricted boltzmann machines for
  multi-label learning with incomplete labels.
\newblock In: AISTATS, pp. 635--643 (2015)

\bibitem{li-au-missing-pr-2016}
Li, Y., Wu, B., Ghanem, B., Zhao, Y., Yao, H., Ji, Q.: Facial action unit
  recognition under incomplete data based on multi-label learning with missing
  labels.
\newblock Pattern Recognition \textbf{60}, 890--900 (2016)

\bibitem{image-tag-missing-cvpr-2013}
Lin, Z., Ding, G., Hu, M., Wang, J., Ye, X.: Image tag completion via
  image-specific and tag-specific linear sparse reconstructions.
\newblock In: CVPR, pp. 1618--1625. IEEE (2013)

\bibitem{information-retrieval-2008}
Manning, C.D., Raghavan, P., Sch{\"u}tze, H., et~al.: Introduction to
  information retrieval, vol.~1.
\newblock Cambridge university press Cambridge (2008)

\bibitem{image-alignment-pami-2012}
Peng, Y., Ganesh, A., Wright, J., Xu, W., Ma, Y.: Rasl: Robust alignment by
  sparse and low-rank decomposition for linearly correlated images.
\newblock IEEE Transactions on Pattern Analysis and Machine Intelligence
  \textbf{34}(11), 2233--2246 (2012)

\bibitem{admm-proof-2014}
Raghunathan, A.U., Di~Cairano, S.: Optimal step-size selection in alternating
  direction method of multipliers for convex quadratic programs and model
  predictive control,”.
\newblock In: Proceedings of Symposium on Mathematical Theory of Networks and
  Systems, pp. 807--814 (2014)

\bibitem{nuclear-norm-2010}
Recht, B., Fazel, M., Parrilo, P.A.: Guaranteed minimum-rank solutions of
  linear matrix equations via nuclear norm minimization.
\newblock SIAM review \textbf{52}(3), 471--501 (2010)

\bibitem{hml-text-icml-2005}
Rousu, J., Saunders, C., Szedmak, S., Shawe-Taylor, J.: Learning hierarchical
  multi-category text classification models.
\newblock In: ICML, pp. 744--751. ACM (2005)

\bibitem{kernel-hml-text-jmlr-2006}
Rousu, J., Saunders, C., Szedmak, S., Shawe-Taylor, J.: Kernel-based learning
  of hierarchical multilabel classification models.
\newblock The Journal of Machine Learning Research \textbf{7}, 1601--1626
  (2006)

\bibitem{mediamill-data-2006}
Snoek, C.G., Worring, M., Van~Gemert, J.C., Geusebroek, J.M., Smeulders, A.W.:
  The challenge problem for automated detection of 101 semantic concepts in
  multimedia.
\newblock In: Proceedings of the 14th annual ACM international conference on
  Multimedia, pp. 421--430. ACM (2006)

\bibitem{ADMM-three-block-convergence-2016}
Sun, H., Wang, J., Deng, T.: On the global and linear convergence of direct
  extension of admm for 3-block separable convex minimization models.
\newblock Journal of Inequalities and Applications \textbf{2016}(1), 227 (2016)

\bibitem{well-multi-label-weak-2010}
Sun, Y., Zhang, Y., Zhou, Z.H.: Multi-label learning with weak label.
\newblock In: AAAI, pp. 593--598 (2010)

\bibitem{hierarchy-image-annotation-review-pr-2012}
Tousch, A.M., Herbin, S., Audibert, J.Y.: Semantic hierarchies for image
  annotation: A survey.
\newblock Pattern Recognition \textbf{45}(1), 333--345 (2012)

\bibitem{bml-cs-active-kdd-2014}
Vasisht, D., Damianou, A., Varma, M., Kapoor, A.: Active learning for sparse
  bayesian multilabel classification.
\newblock In: SIGKDD, pp. 472--481. ACM (2014)

\bibitem{esp-game-2004}
Von~Ahn, L., Dabbish, L.: Labeling images with a computer game.
\newblock In: Proceedings of the SIGCHI conference on Human factors in
  computing systems, pp. 319--326. ACM (2004)

\bibitem{spectral-tutorial-2007}
Von~Luxburg, U.: A tutorial on spectral clustering.
\newblock Statistics and computing \textbf{17}(4), 395--416 (2007)

\bibitem{hash-multi-label-eccv-2014a}
Wang, Q., Shen, B., Wang, S., Li, L., Si, L.: Binary codes embedding for fast
  image tagging with incomplete labels.
\newblock In: ECCV, pp. 425--439. Springer (2014)

\bibitem{hash-multi-label-eccv-2014b}
Wang, Q., Si, L., Zhang, D.: Learning to hash with partial tags: Exploring
  correlation between tags and hashing bits for large scale image retrieval.
\newblock In: ECCV, pp. 378--392. Springer (2014)

\bibitem{kde-nips-2002}
Weston, J., Chapelle, O., Vapnik, V., Elisseeff, A., Sch{\"o}lkopf, B.: Kernel
  dependency estimation.
\newblock In: NIPS, pp. 873--880 (2002)

\bibitem{my-cvpr-2018-d2ia-gan}
Wu, B., Chen, W., Liu, W., Sun, P., Ghanem, B., Lyu, S.: Tagging like humans:
  Diverse and distinct image annotation.
\newblock In: CVPR. IEEE (2018)

\bibitem{my-cvpr-2017-dia}
Wu, B., Jia, F., Liu, W., Ghanem, B.: Diverse image annotation.
\newblock In: CVPR, pp. 2559--2567. IEEE (2017)

\bibitem{my-icpr-2014}
Wu, B., Liu, Z., Wang, S., Hu, B.G., Ji, Q.: Multi-label learning with missing
  labels.
\newblock In: ICPR (2014)

\bibitem{my-iccv-2015}
Wu, B., Lyu, S., Ghanem, B.: Ml-mg: multi-label learning with missing labels
  using a mixed graph.
\newblock In: ICCV, pp. 4157--4165 (2015)

\bibitem{my-aaai-2016-imbalance}
Wu, B., Lyu, S., Ghanem, B.: Constrained submodular minimization for missing
  labels and class imbalance in multi-label learning.
\newblock In: AAAI, pp. 2229--2236 (2016)

\bibitem{my-pr-2015}
Wu, B., Lyu, S., Hu, B.G., Ji, Q.: Multi-label learning with missing labels for
  image annotation and facial action unit recognition.
\newblock Pattern Recognition \textbf{48}(7), 2279--2289 (2015)

\bibitem{tag-completion-pami-2013}
Wu, L., Jin, R., Jain, A.K.: Tag completion for image retrieval.
\newblock IEEE Transactions on Pattern Analysis and Machine Intelligence
  \textbf{35}(3), 716--727 (2013)

\bibitem{multilabel-low-rank-sparse-kdd-2016}
Xu, C., Tao, D., Xu, C.: Robust extreme multi-label learning.
\newblock In: Proceedings of the 22nd ACM SIGKDD International Conference on
  Knowledge Discovery and Data Mining, pp. 13--17 (2016)

\bibitem{MC-speed-nips-2013}
Xu, M., Jin, R., Zhou, Z.H.: Speedup matrix completion with side information:
  Application to multi-label learning.
\newblock In: NIPS, pp. 2301--2309 (2013)

\bibitem{yu-incomplete-hierarchy-bmc-2015}
Yu, G., Zhu, H., Domeniconi, C.: Predicting protein functions using incomplete
  hierarchical labels.
\newblock BMC bioinformatics \textbf{16}(1), 1 (2015)

\bibitem{LEML-ICML-2014}
Yu, H.F., Jain, P., Kar, P., Dhillon, I.: Large-scale multi-label learning with
  missing labels.
\newblock In: ICML, pp. 593--601 (2014)

\bibitem{kronecker-product-1958}
Zehfuss, G.: {\"U}ber eine gewisse determinante.
\newblock Zeitschrift f{\"u}r Mathematik und Physik \textbf{3}, 298--301 (1858)

\bibitem{mlknn-pr-2007}
Zhang, M.L., Zhou, Z.H.: Ml-knn: A lazy learning approach to multi-label
  learning.
\newblock Pattern Recognition \textbf{40}(7), 2038--2048 (2007)

\bibitem{multilabel-review-tkde-2014}
Zhang, M.L., Zhou, Z.H.: A review on multi-label learning algorithms.
\newblock IEEE transactions on knowledge and data engineering \textbf{26}(8),
  1819--1837 (2014)

\bibitem{tianzhu-sparse-coding-eccv-2012}
Zhang, T., Ghanem, B., Liu, S., Ahuja, N.: Low-rank sparse learning for robust
  visual tracking.
\newblock In: ECCV, pp. 470--484. Springer (2012)

\bibitem{multilabel-evaluation-tkdd-2010}
Zhang, Y., Zhou, Z.H.: Multilabel dimensionality reduction via dependence
  maximization.
\newblock ACM Transactions on Knowledge Discovery from Data \textbf{4}(3), 14
  (2010)

\end{thebibliography}

\end{document}